\documentclass{article}

\usepackage{microtype}
\usepackage{subfigure}
\usepackage{booktabs} 

\usepackage{amsthm}         
\usepackage{amsmath}
\usepackage{thmtools,thm-restate} 
\usepackage[utf8]{inputenc} 
\usepackage[T1]{fontenc}    
\usepackage{hyperref}       
\usepackage{url}            
\usepackage{booktabs}       
\usepackage{amsfonts}       
\usepackage{nicefrac}       
\usepackage{microtype}      
\usepackage{multirow}       
\usepackage{threeparttable} 
\usepackage{tabularx}       
\usepackage{graphicx}       
\usepackage{xspace}         
\usepackage[table]{xcolor}  

\usepackage{booktabs}       
\usepackage{array}          
\usepackage{makecell}       

\usepackage{hyperref}       
\usepackage{amsmath}
\usepackage{thmtools,thm-restate}%
\usepackage{todonotes}
\usepackage{subcaption}
\usepackage{lipsum}         
\usepackage{wrapfig}
\usepackage{algorithm}      
\usepackage{algorithmic}

\usepackage{diagbox}        
\usepackage{threeparttable}
\usepackage{colortbl}
\usepackage[most]{tcolorbox} 

\newcommand{\gc}[1]{\cellcolor{gray!35} #1}

\newcommand{\bW}{\mathbf{W}}

\newcommand{\bs}{\mathbf{s}}

\newcommand{\bv}{\mathbf{v}}

\newcommand{\thr}{{\rm{th}}}
\newcommand{\relu}{{\rm ReLU} }

\newcommand{\lup}{{(l)}}

\theoremstyle{definition}  

\definecolor{SkyBlue}{rgb}{0.53, 0.81, 0.92} 

\usepackage{hyperref}

\usepackage[accepted]{arxiv_icml2025}

\usepackage{amsmath}
\usepackage{amssymb}
\usepackage{mathtools}
\usepackage{amsthm}

\usepackage[capitalize,noabbrev]{cleveref}

\icmltitlerunning{Temporal Misalignment and Probabilistic Neurons}

\begin{document}

\twocolumn[
\icmltitle{Temporal Misalignment in ANN-SNN Conversion and \\ Its Mitigation via Probabilistic Spiking Neurons}



\icmlsetsymbol{equal}{*}

\begin{icmlauthorlist}
\icmlauthor{V. X. BojWu}{equal,amal}
\icmlauthor{Velibor Bojkovi\'c}{equal,mbzuai}
\icmlauthor{Xiaofeng Wu}{equal,citym}
\icmlauthor{Bin Gu}{juni}
\end{icmlauthorlist}

\icmlaffiliation{amal}{Amalgamation of first authors' names}
\icmlaffiliation{mbzuai}{Department of ML, MBZUAI, Abu Dhabi, UAE}
\icmlaffiliation{citym}{Faculty of Data Science, City University of Macau, Macau, China}
\icmlaffiliation{juni}{School of Artificial Intelligence, Jilin University, China}

\icmlcorrespondingauthor{Velibor Bojkovi\'c}{first.last@mbzuai.ac.ae}


\vskip 0.3in
]



\printAffiliationsAndNotice{\icmlEqualContribution} 

\begin{abstract}

Spiking Neural Networks (SNNs) offer a more energy-efficient alternative to Artificial Neural Networks (ANNs) by mimicking biological neural principles, establishing them as a promising approach to mitigate the increasing energy demands of large-scale neural models. However, fully harnessing the capabilities of SNNs remains challenging due to their discrete signal processing and temporal dynamics. ANN-SNN conversion has emerged as a practical approach, enabling SNNs to achieve competitive performance on complex machine learning tasks. In this work, we identify a phenomenon in the ANN-SNN conversion framework, termed \textit{temporal misalignment}, in which random spike rearrangement across SNN layers leads to performance improvements. Based on this observation, we introduce biologically plausible two-phase probabilistic (TPP) spiking neurons, further enhancing the conversion process. We demonstrate the advantages of our proposed method both theoretically and empirically through comprehensive experiments on CIFAR-10/100, CIFAR10-DVS, and ImageNet across a variety of architectures, achieving state-of-the-art results.
\end{abstract}

\section{Introduction}
\label{sec:intro}

\begin{figure}
\centering
\includegraphics[width=\linewidth]{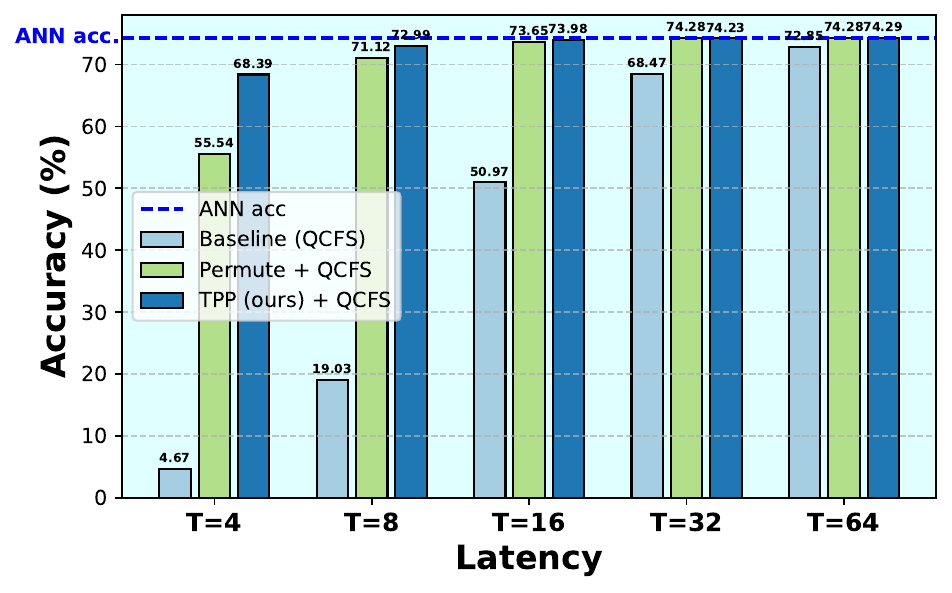}
\caption{\footnotesize The initial experiment: After ANN-SNN conversion, we compared the accuracy of the baseline model QCFS \citep{bu2022optimal} with its ``permuted'' version and our proposed TPP neurons (setting is VGG16 - ImageNet, ANN acc. 74.29\%).}
\label{fig-init-exps}
\vspace{-10pt}
\end{figure}
\vspace{-2pt} 

Spiking neural networks (SNNs), often referred to as the third generation of neural networks~\cite{maas1997third}, closely mimic biological neuronal communication through discrete spikes~\cite{mcculloch1943logical,hodgkin1952quantitative,izhikevich2003simple}. While biological energy efficiency has historically influenced neural network design, SNNs uniquely achieve this through event-driven processing: weighted inputs integrate into membrane potentials, emitting binary spikes only upon crossing activation thresholds. This differs significantly from artificial neural networks (ANNs)~\cite{braspenning1995artificial}, which are based on continuous floating-point operations that require energy-intensive and computationally costly multiplication operations~\cite{roy2019towards}. The spike-driven paradigm inherently circumvents these costly computations, suggesting that SNNs may offer a promising approach for energy-efficient AI. Recent developments in neuromorphic hardware~\cite{pei2019towards, debole2019truenorth, loihi2, DBLP:journals/corr/abs-2312-17582} have enabled efficient SNN deployment. These specialized chips are inherently designed for spike-based computation, driving breakthroughs across domains: object detection~\cite{kim2020spiking, cheng2020lisnn}, tracking~\cite{yang2019dashnet}, event-based vision~\cite{zhu2022event, DBLP:journals/iclr/SpikePoint}, speech recognition~\cite{DBLP:conf/aaai/WangZHWZX23}, and generative AI through models like SpikingBERT~\cite{bal2023spikingbert} and SpikeGPT~\cite{DBLP:journals/corr/abs-2302-13939, wang2023masked}. These advancements underscore the potential of SNNs as a viable alternative to conventional ANNs.

Training SNNs is inherently challenging due to the same characteristics that confer their advantages: their discrete processing of information. Unsupervised direct training, inspired by biological learning mechanisms, leverages local learning rules and spike timing to update weights~\cite{diehl2015unsupervised}. While these methods are computationally efficient and can be executed on specialized hardware, SNNs trained in this manner often underperform compared to models trained with alternative approaches. Further research is needed to better understand and improve this method. In contrast, supervised training can be categorized into direct training via spatio-temporal backpropagation (e.g., surrogate gradient methods~\cite{neftci2019surrogate}) and ANN-SNN conversion methods~\cite{diehl2015fast, cao2015spiking}. The present work focuses on the latter approach.

\begin{figure}[!t]
  \centering
  \includegraphics[width=0.9\linewidth]{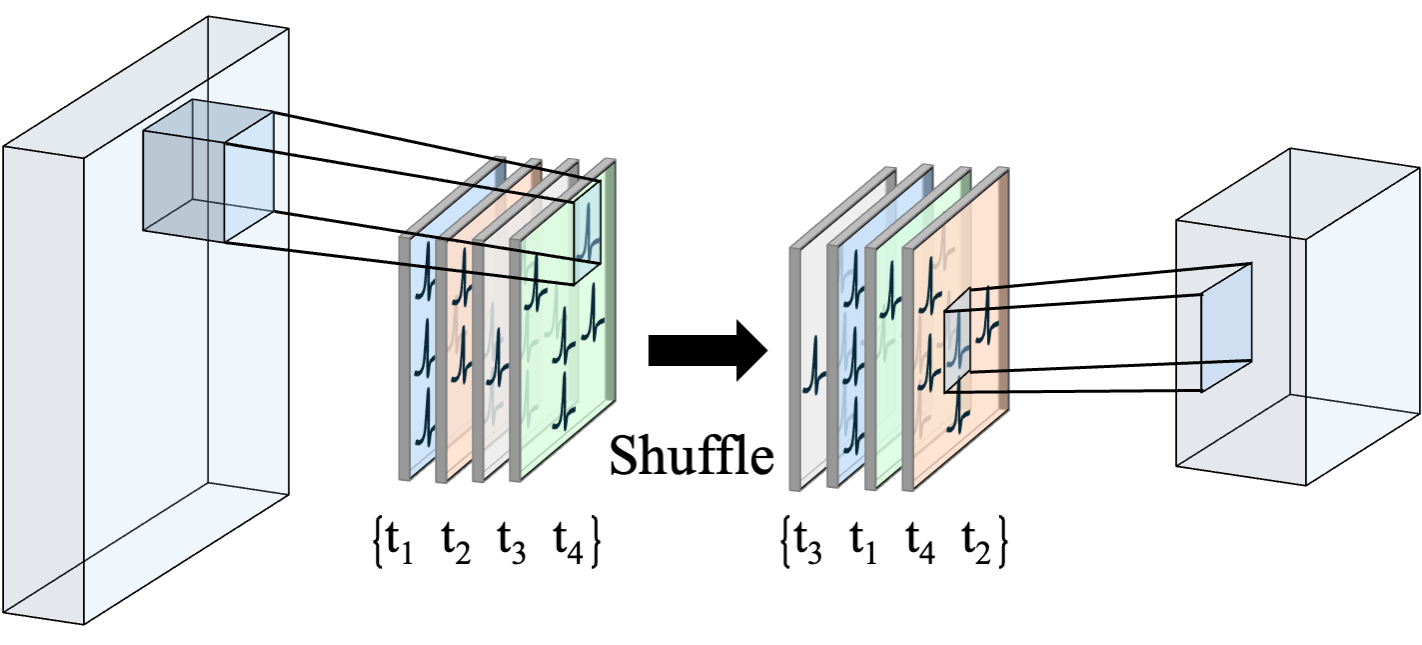}
  \caption{Spike train permutation: spikes at different time steps are shuffled to alter their temporal order.}
  \label{fig:permutation}
\end{figure}

The core idea of ANN-SNN conversion is to leverage pre-trained ANN models to train SNNs. This process begins by transferring the weights from the ANN to the SNN, which shares the same architecture, and initializing the spiking neuron parameters (e.g., thresholds, time constants) so that the spike rates approximate the activation values of the corresponding ANN layers. This method is advantageous because it typically requires no additional computation for training the SNN, thereby eliminating the need for gradient computation or limiting it to fine-tuning the SNN.

\textbf{The initial experiment: Temporal information in SNNs after ANN-SNN conversion.} We begin by identifying a counter-intuitive phenomenon that, to the best of our knowledge, has not been previously documented. We investigate the extent to which individual spike timing affects the overall performance of the SNN model in ANN-SNN conversion. We explore this question in the context of various types of conversion-related issues described in the literature, such as \textit{phase lag}~\cite{li2022bsnn} and \textit{unevenness of the spike input}~\cite{bu2022optimal}, where prior work has indicated that the timing of spike trains in SNNs obtained through ANN-SNN conversion may not be optimal and can lead to performance degradation under low-latency conditions.

\begin{figure}[!ht]
  \centering
  \includegraphics[width=0.95\linewidth]{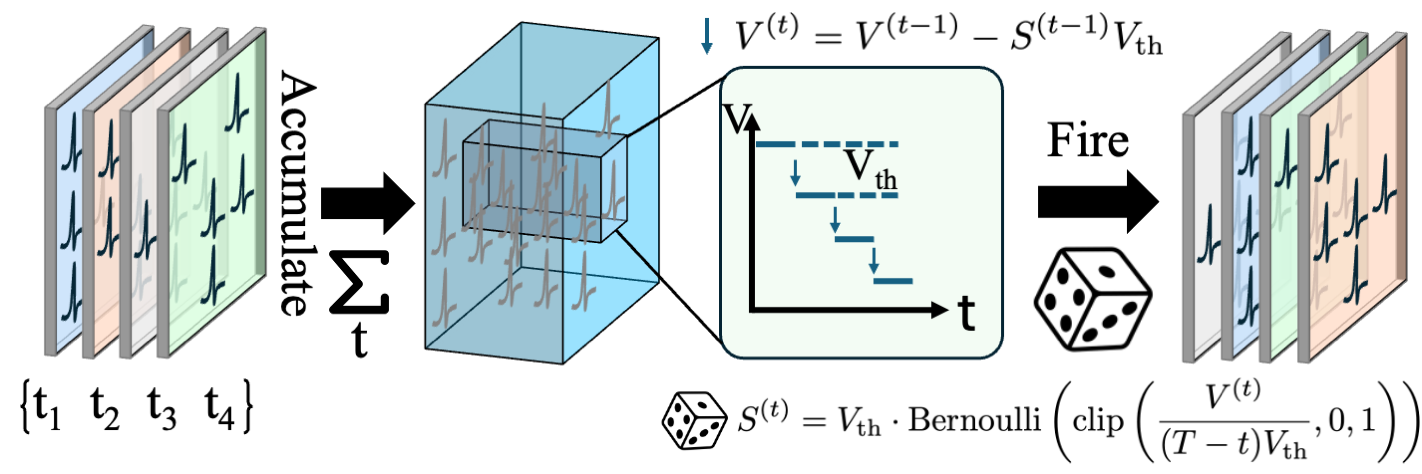}
  \caption{Two-Phase Probabilistic (TPP) spiking neuron mechanism operates in two phases: accumulation of inputs and probabilistic spiking based on membrane potential. The model uses a Bernoulli process for spike generation, with membrane potential updates over time steps.}
  \label{fig:prob-mode}
  \vspace{-5pt}
\end{figure}

In our initial experiment, we began by examining several widely-used ANN-SNN conversion methods and analyzing the spike outputs of spiking layers in the corresponding SNN models. To evaluate the importance of spike timing relative to firing rates, we randomly permuted the spike trains after each spiking layer. Specifically, when processing samples through the baseline model, we rearranged the temporal order of spikes within each spike train after each spiking layer (see Figure \ref{fig:permutation}). The permuted spike trains were then passed to the subsequent layer, and this process was repeated for each layer until the output layer. The results of one of these initial experiments, comparing the performance of the ``permuted'' model with the baseline, are presented in Figure~\ref{fig-init-exps}. For each latency, we conducted multiple trials with different permutations, and the ``permuted'' model consistently outperformed the baseline, achieving the original ANN accuracy at lower latencies. The impact of permutations on performance was particularly pronounced at lower latencies.

We refer to this phenomenon as ``\textit{temporal misalignment}'' in ANN-SNN conversion and further investigate it by providing a conceptual interpretation in the form of bursting probabilistic spiking neurons, which are designed to mimic the effects of permutations in SNNs. The proposed neurons operate in two phases, as illustrated in Figure~\ref{fig:prob-mode}. In the first phase, they accumulate input, often surpassing the threshold, while in the second phase, they emit spikes probabilistically with varying temporal probabilities. Our proposed spiking neurons are characterized by two key properties: (1) the accumulation of membrane potential beyond the threshold, resulting in a firing phase, commonly referred to as bursting, and (2) probabilistic firing. Both properties exhibit biological plausibility and have been extensively studied in the neuroscience literature (see Section~\ref{sec bio inspired}).

The main contributions of this paper are summarized as:

$\bullet{}$ We recognize and study the ``temporal misalignment'' phenomenon in ANN-SNN conversion, and we propose a framework for its exploitation in ANN-SNN conversion utilizing two-phase probabilistic spiking neurons. We provide theoretical insights into their functioning and superior performance, as well as support for their biological grounding.

$\bullet{}$ We present a comprehensive experimental validation demonstrating that our method outperforms SOTA conversion as well as the other training methods in terms of accuracy on CIFAR-10/100, CIFAR10-DVS, and ImageNet datasets.

\section{Background and Related Work} \label{sec:preliminaries} 
The base model employed in this work is Integrate-and-Fire (IF) spiking neuron, whose internal dynamics, after discretization, are given by the equations
\begin{align}\label{eq: dynamics}
\bv^\lup[t] &= \bv^\lup[t-1] +\bW^\lup \theta^{(l-1)}\cdot \bs^{(l-1)}[t]-\theta^\lup \cdot\bs[t-1],\\
\bs^\lup[t] &= H(\bv^\lup[t]- \theta^\lup).
\end{align}
Here, $\theta^\lup$ is the threshold, $H(\cdot)$ is the Heaviside function, while the superscript $l$ denotes the layer in the SNN. Later, we will modify these equations and use more advanced neuron models. By expanding the equations over $t=1,\dots,T$, and rearranging the terms, we obtain
\begin{align}
 \theta^\lup \frac{\sum_{t=1}^T\bs^{(l)}[t]}{T} &= \bW^{(l)}V^{(l-1)}_\thr\frac{\sum_{t=1}^T\bs^{(l-1)}[t]}{T}\label{eq: rel}\\
 &+\frac{\bv^{(l)}[T] - \bv^{(l)}[0]}{T}.\label{eq: rel err}
 \end{align}
On the ANN side, the transformation between layers is given by 
\begin{equation}\label{eq: ann eq}
a^{(l)}= \mathcal{A}^{(l)}(\bW^{(l)} a^{(l-1)}),
\end{equation}
where $\mathcal{A}^\lup$ is the activation function. The ANN-SNN conversion process begins by transferring the weights (and biases) of a pre-trained ANN to an SNN with the same architecture. By comparing the equations for the ANN outputs \eqref{eq: ann eq} and the average output of the SNN \eqref{eq: rel} \cite{rueckauer2016theory}, the goal is to achieve a relation of the form
 \begin{equation}\label{eq: ann snn}
 a^{(l)}_i \approx  V^{(l)}_\thr \frac{\sum_{t=1}^T\bs^{(l)}_i[t]}{T}.
 \end{equation}
The most commonly used activation function $\mathcal{A}$ is \relu, due to its simplicity and non-negative output, which aligns well with the properties of IF neurons. It is important to note the importance of the three components in the conversion: 1) the threshold value $\theta$, 2) the initialization $\bv[0]$, 3) the ANN activation function $\mathcal{A}$.

\subsection{Related work}

\textbf{ANN-SNN Conversion}. This methodology aligns ANNs and SNNs through activation-firing rate correspondence, as initially demonstrated in~\cite{rueckauer2016theory, cao2015spiking}. Subsequent research has systematically improved conversion fidelity through four principal approaches: (i) weight normalization~\cite{diehl2015fast}, (ii) soft-reset mechanisms~\cite{rueckauer2017conversion, han2020rmp}, (iii) adaptive threshold configurations~\cite{stockl2021optimized, ho2021tcl, DBLP:journals/tnn/WuCZLLT23}, and (iv) spike coding optimization~\cite{DBLP:conf/nips/KimKK20a, Sengupta2018Going}. Recent innovations focus on ANN activation function adaptations, including thresholded ReLU~\cite{ding2021optimal} and quantization strategies~\cite{bu2022optimal, liu2022spikeconverter, hu2023fast, shen2024conventional}. However, these approaches introduce inherent accuracy-compression tradeoffs. Parallel efforts modify integrate-and-fire neuron dynamics~\cite{DBLP:conf/ijcai/Li022, wang2022signed, liu2022spikeconverter}, with~\cite{liu2022spikeconverter} proposing a dual-phase mechanism resembling our approach.

Crucial for achieving high conversion efficacy, threshold initialization methodologies employ layer-wise activation maxima or percentile statistics~\cite{rueckauer2016theory, DBLP:conf/iclr/DengG21, li2021free, wu2024ftbc}, augmented with post-conversion weight calibration~\cite{li2021free, bojkovic2024data}. Contemporary strategies can be categorized into two paradigms: (1) ANN activation quantization for temporal efficiency at the cost of accuracy, and (2) SNN neuron modification preserving ANN expressiveness with extended temporal requirements. Our methodology adheres to the second paradigm.

\textbf{Direct Training}. This approach leverages spatio-temporal spike patterns through backpropagation-through-time with differentiable gradient approximations~\cite{DBLP:conf/iclr/OConnorGRW18, zenke2018superspike, wu2018spatio, bellec2018long, fang2021deep, fang2021incorporating, zenke2021remarkable, mukhoty2024direct}. Advancements encompass joint optimization of synaptic weights and neuronal parameters (threshold dynamics~\cite{wei2023temporal}, leakage factors~\cite{DBLP:journals/tnn/RathiR23}), novel loss formulations for spike distribution regularization~\cite{zhu2024exploring, guo2022recdis}, and hybrid conversion-training pipelines~\cite{DBLP:journals/corr/abs-2205-07473}. State-of-the-art developments introduce ternary spike representations for enhanced information density~\cite{DBLP:conf/aaai/GuoCLPZHM24} and reversible architectures for memory-efficient training~\cite{zhang2023memory}.

\section{Methodology} 
\label{sec:method}

\subsection{Motivation}

In ANN-SNN conversion methodologies, constant and rate coding are commonly utilized in the resulting spiking neural network models, based on the principle that the expected input at each time step matches the original input to the ANN model. Notably, the encoding lacks temporal information, as spike timing does not convey additional information. In constant encoding, this is evident, while in rate encoding, for a fixed input channel, the spike probability remains constant across all time steps, with the channel value assumed to lie between 0 and 1.

The resulting SNN model is initialized to approximate the outputs of the original ANN model based on the principle that, for each spiking neuron, the expected number of spikes it generates should closely match the output of the corresponding ANN neuron. In particular, it is assumed that no temporal information is present throughout the SNN model; that is, the spike train outputs of each SNN layer should convey no additional temporal information beyond spike firing rates.

Previous studies examining conversion errors and their classifications \cite{li2022bsnn, bu2022optimal, ding2021optimal, bojkovic2024data} suggest that SNNs obtained through ANN-SNN conversion may generate spikes that are suboptimally positioned in the temporal domain, leading to a degradation in model performance, particularly at low latencies. Our initial experiments (see Introduction and Figure~\ref{fig-init-exps}) further validate this observation while also revealing a novel insight: random temporal displacements of spike trains after spiking layers significantly enhance model performance. This phenomenon, which we refer to as temporal misalignment -- wherein the original spike trains exhibit temporal misalignment, thereby impairing model performance -- serves as the foundation and motivation for our proposed method, which is elaborated upon in the next section. Additional experiments on permutations in the ANN-SNN context, along with an explanation of their impact on model performance, are provided in Appendix \ref{app permutations}.

\subsection{From permutations to Bursting Probabilistic Spiking Neurons}\label{sec latency}


This work aims to address the following question: How to incorporate the action of permutation of the output spike trains into the dynamics of the spiking neurons? We approach this problem in two steps.


Consider the scenario where spike trains from layer $\ell$ are to be permuted. A general approach involves introducing a ``permutator'' -- a subsequent layer tasked with collecting all incoming spikes and re-emitting them in a permuted manner, as illustrated in Figure \ref{fig:permutation}. This inherently implies the \textbf{two-phase} nature of the ``permutator'': specifically, during the first phase, incoming spikes are accumulated, and firing is deferred until the onset of the second phase, where spikes are emitted.

The second step focuses on the output mechanism of the ``permutator''. Specifically, it is desirable to design a spiking neuron mechanism that retains the stochastic component of the permutations. This consideration motivates the adoption of probabilistic firing in spiking neurons.


The final question we explore is whether a more compact approach can be achieved by employing probabilistic spiking neurons that aggregate weighted inputs from the previous layer, rather than directly processing spikes from a spiking layer.

\textbf{TPP neurons}: To address the aforementioned questions, we propose \textbf{two-phase probabilistic spiking neurons} (TPP) (see Figure \ref{fig:prob-mode}). Specifically, in the first phase, the neurons will only accumulate the (weighted) input coming from the previous layer, while in the second phase, the neurons will spike. More precisely, suppose that at a particular layer $\ell$ the spiking neurons accumulate the whole output of the previous layer, without emitting spikes. Denote the accumulated membrane potential by $\bv^\lup[0]$. The subsequent spiking phase is governed by the following equations:
\begin{equation}\label{eq probabilistic spiking}
\begin{aligned}
\bs^\lup[t] &= B\left(\frac{1}{\theta^\lup\cdot (T-t+1)}\cdot\bv^\lup[t-1]\right),\\
\bv^\lup[t] &= \bv^\lup[t-1] -\theta^\lup \cdot\bs^\lup[t],
\end{aligned}
\end{equation}
where $t=1,\dots, T$. Here, $B(x)$ is a Bernoulli random variable with bias $x$, extended for $x\in \mathbb{R}$ in a natural way ($B(x)= B(\max(\min(x,1),0))$). If the weights of the SNN network are not normalized, the produced spikes will be scaled with the thresholds $\theta^\lup\cdot \bs^\lup[t]$, before being sent to the next layer.

We observe that the presence of $T-t+1$ in the denominator of the bias in $B$ demonstrates that the probability of spiking depends not only on the current membrane potential, but also on the temporal step: in the absence of spiking, for the same membrane potential, the probability of spiking increases through time. Figure \ref{fig:prob-mode} provides a visual representation of the functioning of TPP neurons.





The following theorem provides a comprehensive characterization of the functioning of TPP neurons and their applications in ANN-SNN conversion when approximating ANN outputs (here $\relu_\theta(x)=\min(\relu(x),\theta)$). 

\begin{restatable}{theorem}{mainthm}\label{thm main}
    Let $X^{(l)}$ be the input of the ANN layer with \relu activation and suppose that, during the accumulation phase, the corresponding SNN layer of TPP neurons accumulated $T\cdot X^{(l)}$ quantity of voltage. \\
        (a) For every time step $t=1,\dots,T$, we have         
        \begin{equation}
            \frac{(T-t+1)\cdot\theta^\lup}{T}\cdot \mathbb{E}\left[\sum_{i=1}^t s^{(l)}[i]\right] = \relu_{\theta^\lup}(X^{(l)}).
        \end{equation}        
        (b) Suppose that for some $t=1,\dots,T$, the TPP layer produced $s^{(l)}[1],\dots,s^{(l)}[t-1]$ vector spike trains for the first $t-1$ steps, and the residue voltage for neuron $i$ is higher than zero. Then,
        \begin{equation}
            \frac{(T-t+1)\cdot\theta^\lup}{T}\cdot\mathbb{E}\left[s_i^{(l)}[t]\right]+\frac{\theta^\lup}{T}\cdot\sum_{i=1}^{t-1} s_i^{(l)}[i] = \relu_{\theta^\lup}(X_i^{(l)}).
        \end{equation}        
        (c) If $s^{(l)}[1],\dots,s^{(l)}[T]$ are the output vectors of spike trains of the TPP neurons during $T$ time steps, then        
        \begin{equation}
            \frac{\theta^\lup}{T}\cdot\sum_{i=1}^{T} s^{(l)}_j[i]=\relu_{\theta^\lup}(X^{(l)}_j),
        \end{equation}
        if $\relu_{\theta^\lup}(X^{(l)}_j)$ is a multiple of $\frac{\theta^\lup}{T}$, or        
        \vspace{-5pt}
        \begin{align}
            \frac{\theta^\lup}{T}\cdot\sum_{i=1}^{T} s^{(l)}_j[i]&=\frac{\theta^\lup}{T}\cdot \lfloor\frac{T}{\theta^\lup}\relu_{\theta^\lup}(X^{(l)}_j)\rfloor\label{eq out 1}\\
            &\text{ or } \frac{\theta^\lup}{T}\cdot \lfloor\frac{T}{\theta^\lup}\relu_{\theta^\lup}(X^{(l)}_j)\rfloor+\frac{\theta^\lup}{T},\label{eq out 2}
        \end{align}        
        if $\relu_{\theta^\lup}(X^{(l)}_j)$ is not a multiple of $\frac{\theta^\lup}{T}$.
        
        $(d)$ Suppose that $\max X^{(l)}\leq \theta$ and that the same weights $W^{(l+1)}$ act on the outputs of layer $(l)$ of ANN and SNN as above, and let $X^{(l+1)}$ (resp. $T\cdot \tilde{X}^{(l+1)}$) be the inputs to the $(l+1)$th ANN layer (resp. the accumulated voltage for the $(l+1)$th SNN layer of TPP neurons), Then 
        \begin{equation}\label{eq final approximation}
            ||X^{(l+1)}-\tilde{X}^{(l+1)}||_{\infty} \leq ||W^{(l+1)}||_\infty\cdot \frac{\theta^\lup}{T}.
        \end{equation} 
        \vspace{-5pt}
\end{restatable}
\textbf{Remarks:} The proof of the previous result is presented in the Appendix. Here, we offer an interpretation of its statements.

    $\bullet{}$ We contrast the statement $(a)$ with Theorem 2 of \cite{bu2022optimal}. Specifically, the authors demonstrate that if the membrane potential is initialized at half of the threshold, the expectation of the conversion error (layerwise) is 0. However, this result in \cite{bu2022optimal} relies on the underlying assumption that the layerwise distribution of ANN activation values is \textbf{uniform}, which does not generally hold in practice (see, for example, \cite{bojkovic2024data}). Our result $(a)$ above shows that \textbf{after every $t\leq T$ time steps}, our expected spiking rate aligns well with the clipping of the $\relu$ activation by the threshold, as it should, without imposing any prior assumptions on the distribution of the ANN activation values.
    
    $\bullet{}$ Result $(b)$ demonstrates that the activity of a TPP neuron \textbf{adapts to the output} it already produced. In particular, as long as the neuron is still active and contains residual membrane potential, the expectation of its output at the next time step takes into account the previously produced spikes and will yield the ANN counterpart.
    
    $\bullet{}$ The results $(c)$ and $(d)$ show that during the accumulation phase, the TPP neuron closely approximate ANN neurons with \relu activation. In particular, the only remaining source of errors in layerwise approximation is the clipping error due to the threshold $\theta$, and the quantization error due to the discrete outputs of the spiking neurons. We also note in equations \eqref{eq out 1} and \eqref{eq out 2} two possibilities for the output, which come from the probabilistic nature of spiking.

\subsection{Bio-plausibility and hardware implementation of TPP neurons}\label{sec bio inspired}
Our proposed neurons have two distinct properties: The two-phase regime and probabilistic spike firing. Both properties are biologically plausible and extensively studied in the neuroscience literature. For example, the two phase regime is related to firing after a delay of biological spiking neurons, where a neuron collects the input beyond the threshold value and fires after delay or after some condition is met. It could also be related to the bursting, when a biological neuron starts emitting bursts of spikes, after a certain condition is met, effectively dumping their accumulated potential. See \cite{izhikevich2007dynamical, connors1990intrinsic,llinas1982electrophysiology,krahe2004burst} for more details. 

On the other side, stochastic firing of biological neurons has been well studied as well, and different aspects of noise introduction into firing have been proposed. Refer to \cite{shadlen1994noise,faisal2008noise,softky1993highly,maass1997networks,pagliarini2019probabilistic,stein2005neuronal} for some examples.

Regarding the implementation of TPP neurons on neuromorphic hardware, two phase regime can be easily achieved on various modern neuromorphic that support programmable spiking neurons. Stochastic firing can be achieved through random sampling which, for example, is supported by IBM TrueNorth \cite{merolla2014million}, Intel Loihi \cite{davies2018loihi}, BrainScaleS-2 \cite{pehle2022brainscales}, SpiNNaker \cite{furber2014spinnaker} neuromorphic chips. For example, TrueNorth incorporates stochastic neuron models using on-chip pseudo-random number generators, enabling probabilistic firing patterns that mirror our approach. Similarly, Loihi~\cite{gonzalez2024spinnaker2} supports stochastic operations by adding uniformly distributed pseudorandom noise to neuronal variables, facilitating the implementation of probabilistic spiking neurons.

To reduce the overall latency for processing inputs with our models, which yields linear dependence on the number of layers (implied by the two phase regime), we note that as soon as a particular layer has finished the firing phase, it can start receiving the input from the previous one: The process of classifying a dataset can be serialized. This has already been observed, for example in \cite{liu2022spikeconverter}. Neuromophic hardware implementation of this serialization has been proposed as well, see for example \cite{das2023design, song2021design, varshika2022design}.

\section{Experiments} 

In this section, we verify the effectiveness and efficiency of our proposed methods. We compare it with state-of-the-art methods for image classification via converting ResNet-20, ResNet-34~\cite{he2016residual}, VGG-16~\cite{simonyan2015deep}, RegNet~\cite{DBLP:conf/cvpr/RadosavovicKGHD20} on CIFAR-10~\cite{lecun1998gradient, krizhevsky2010cifar}, CIFAR-100~\cite{krizhevsky2009learning}, CIFAR10-DVS~\cite{li2017cifar10} and ImageNet~\cite{deng2009imagenet}. Our experiments use PyTorch~\cite{paszke2019pytorch}, PyTorch vision models~\cite{torchvision2016}, and the PyTorch Image Models (Timm) library~\cite{rw2019timm}.

To demonstrate the wide applicability of the TPP neurons and the framework we propose, we combine them with three representative methods of ANN-SNN conversion from recent literature, each of which has their own particularities. These methods are: QCFS~\cite{DBLP:conf/iclr/BuFDDY022}, RTS~\cite{DBLP:conf/iclr/DengG21}, and SNNC~\cite{li2021free}. The particularity of QCFS method is that it uses step function instead of \relu in ANN models during their training, in order to obtain higher accuracy in lower latency after the conversion. RTS method uses thresholded \relu activation in ANN models during their training, so that the outliers are eliminated among the activation values, which helps to reduce the conversion error. Finally, SNNC uses standard ANN models with \relu activation, and performs grid search on the activation values to find optimal initialization of the thresholds in the converted SNNs. 

We initialize our SNNs following the standard ANN-SNN conversion process described in Section~\ref{sec:method} (and detailed in \ref{app conversion steps}), starting with a pre-trained model given by the baseline, or with training an ANN model using default settings in QCFS~\cite{DBLP:conf/iclr/BuFDDY022}, RTS~\cite{DBLP:conf/iclr/DengG21}, and SNNC~\cite{li2021free}. ANN ReLU activations were replaced with layers of TPP neurons initialized properly. All experiments were conducted using NVIDIA RTX 4090 and Tesla A100 GPUs. For comprehensive details on all setups and configurations, see Appendix~\ref{appendix:config}.

\subsection{Comparison with the State-of-the-art ANN-SNN Conversion methods}

We evaluate our approach against previous state-of-the-art ANN-SNN conversion methods, including ReLU-Threshold-Shift (RTS)~\cite{DBLP:conf/iclr/DengG21}, SNN Calibration (SNNC-AP)~\cite{li2021free}, Quantization Clip-Floor-Shift activation function (QCFS)~\cite{DBLP:conf/iclr/BuFDDY022}, SNM~\cite{wang2022signed}, Burst~\cite{DBLP:conf/ijcai/Li022}, OPI~\cite{bu2022optimized}, SRP~\cite{DBLP:conf/aaai/HaoBD0Y23}, DDI~\cite{bojkovic2024data} and FTBC~\cite{wu2024ftbc}.

\noindent \textbf{ImageNet dataset:} Table~\ref{tab:ann-snn-imagenet-cifar} compares the performance of our proposed methods with state-of-the-art ANN-SNN conversion methods on ImageNet. Our method outperforms the baselines across all simulation time steps for VGG-16, and RegNetX-4GF. For instance, on VGG-16 at $T=32$, our method achieves 74.72\% accuracy, surpassing other baselines even at $T=128$. Moreover, at $T=128$, our method nearly matches the original ANN performance with only a 0.12\% drop in VGG-16 and a 0.14\% drop in ResNet-34.

We see similar patterns in combining our methods with RTS and QCFS baselines, which use modified \relu activations to reduce conversion errors. Table \ref{tab:ann-snn-imagenet-cifar} shows these results. For instance, applying TPP with QCFS on ResNet-34 at $T=16$ improves performance from 59.35\% to 72.03\%, a 12.68\% increase. Similarly, for VGG-16 at $T=16$, combining TPP with QCFS boosts performance from 50.97\% to 73.98\%, a 23.01\% increase. Using TPP with RTS also shows significant improvements, such as a 12.82\% increase for VGG-16 at $T=16$. These results demonstrate the benefits of integrating TPP with other optimization approaches, solidifying its role as a comprehensive solution for ANN-SNN conversion challenges.

\begin{table*}[ht]
\vspace{-10pt}
\caption{Comparison between our method and the other ANN-SNN conversion methods on ImageNet and CIFAR-100. We provide the average accuracy and the associated standard deviation across 5 experiments.}
\label{tab:ann-snn-imagenet-cifar}
\centering
\scalebox{0.7}{
\begin{threeparttable}
\begin{tabular}{@{}cccccccccc@{}}
\toprule
Dataset & Arch. & Method & ANN & T=4  & T=8 & T=16 & T=32 & T=64 & T=128 \\
\toprule
\multirow{23}{*}{ImageNet} 
& \multirow{8}{*}{ResNet-34}
& RTS~\cite{DBLP:conf/iclr/DengG21}\textsuperscript{ICLR} & 75.66 & -- & -- & -- & 33.01 & 59.52 & 67.54 \\
& & SNNC-AP\tnote{*}~\cite{li2021free}\textsuperscript{ICML} & 75.66 & -- & -- & -- & 64.54 & 71.12 & 73.45 \\
& & QCFS~\cite{DBLP:conf/iclr/BuFDDY022}\textsuperscript{ICLR} & 74.32 & -- & -- & 59.35 & 69.37 & 72.35 & 73.15 \\
& & SRP~\cite{DBLP:conf/aaai/HaoBD0Y23}\textsuperscript{AAAI} & 74.32 & \gc{66.71} & \gc{67.62} & 68.02 & 68.40 & 68.61 & --\\
& & FTBC(+QCFS)~\cite{wu2024ftbc}\textsuperscript{ECCV} & 74.32 & 49.94 & 65.28 & 71.66 & 73.57 & 74.07 & 74.23 \\
\cmidrule{3-10}
& & \textbf{Ours (TPP) + QCFS} & 74.32 & {37.23 (0.07)} & {67.32 (0.06)} & \gc{72.03 (0.02)} & 72.97 (0.03) & 73.24 (0.02) & 73.30 (0.02) \\
\cmidrule{3-10}
& & \textbf{Ours (TPP)\tnote{*} + SNNC w/o Cali.} & 75.65 & 2.69 (0.03) & 49.24 (0.23) & 69.97 (0.10) & \gc{74.07 (0.06)} & \gc{75.23 (0.03)} & \gc{75.51 (0.05)} \\
\cmidrule{2-10}
& \multirow{12}{*}{VGG-16}
& SNNC-AP\tnote{*}~\cite{li2021free}\textsuperscript{ICML} & 75.36 & -- & -- & -- & 63.64 & 70.69 & 73.32 \\
& & SNM\tnote{*}~\cite{wang2022signed}\textsuperscript{IJCAI} & 73.18 & -- & -- & -- & 64.78 & 71.50 & 72.86 \\
& & RTS~\cite{DBLP:conf/iclr/DengG21}\textsuperscript{ICLR} & 72.16 & -- & -- & 55.80 & 67.73 & 70.97 & 71.89 \\
& & QCFS~\cite{DBLP:conf/iclr/BuFDDY022}\textsuperscript{ICLR} & 74.29 & -- & -- & 50.97 & 68.47 & 72.85 & 73.97 \\
& & Burst~\cite{DBLP:conf/ijcai/Li022}\textsuperscript{IJCAI} & 74.27 & -- & -- & -- & 70.61 & 73.32 & 73.00 \\
& & OPI\tnote{*}~\cite{bu2022optimized}\textsuperscript{AAAI} & 74.85 & -- & 6.25 & 36.02 & 64.70 & 72.47 & 74.24 \\
& & SRP~\cite{DBLP:conf/aaai/HaoBD0Y23}\textsuperscript{AAAI} & 74.29 & 66.47 & 68.37 & 69.13 & 69.35 & 69.43 & --\\
& & FTBC(+QCFS)~\cite{wu2024ftbc}\textsuperscript{ECCV} & 73.91 & 58.83 & 69.31 & 72.98 & 74.05 & 74.16 & 74.21 \\
\cmidrule{3-10}
& & \textbf{Ours (TPP) + RTS} & 72.16 & 30.50 (1.19) & 56.69(0.67) & 67.34 (0.25) & 70.63 (0.11) & 71.75 (0.05) & 72.05 (0.03) \\
\cmidrule{3-10}
& & \textbf{Ours (TPP) + QCFS} & 74.22 & \gc{68.39 (0.08)} & \gc{72.99 (0.05)} & \gc{73.98 (0.07)} & 74.23 (0.03) & 74.29 (0.00) & 74.33 (0.01) \\
\cmidrule{3-10}
& & \textbf{Ours (TPP)\tnote{*} + SNNC w/o Cali.} & 75.37 & 54.14 (0.59) & 69.75 (0.27) & 73.44 (0.02) & \gc{74.72 (0.06)} & \gc{75.14 (0.02)} & \gc{75.25 (0.03)} \\
\cmidrule{2-10}
& \multirow{3}{*}{RegNetX-4GF}
& RTS~\cite{DBLP:conf/iclr/DengG21}\textsuperscript{ICLR} & 80.02 & -- & -- & -- & 0.218 & 3.542 & 48.60 \\
& & SNNC-AP\tnote{*}~\cite{li2021free}\textsuperscript{ICML} & 80.02 & -- & -- & -- & 55.70 & 70.96 & 75.78 \\ 
\cmidrule{3-10}
& & \textbf{Ours (TPP)\tnote{*} + SNNC w/o Cali.} & 78.45 & -- & -- & \gc{22.71 (2.98)} & \gc{66.51 (0.44)} & \gc{75.54 (0.07)} & \gc{77.83 (0.04)} \\
\midrule
\multirow{22}{*}{CIFAR-100}
& \multirow{10}{*}{ResNet-20}
& TSC\tnote{*}~\cite{han2020deep}\textsuperscript{ECCV} & 68.72 & -- & -- & -- & -- & -- & 58.42 \\
& & RMP\tnote{*}~\cite{han2020rmp}\textsuperscript{CVPR} & 68.72 & -- & -- & -- & 27.64 & 46.91 & 57.69 \\
& & SNNC-AP\tnote{*}~\cite{li2021free}\textsuperscript{ICML} & 77.16 & -- & -- & 76.32 & 77.29 & 77.73 & 77.63 \\
& & RTS~\cite{DBLP:conf/iclr/DengG21}\textsuperscript{ICLR} & 67.08 & -- & -- & 63.73 & 68.40 & 69.27 & 69.49 \\
& & OPI\tnote{*}~\cite{bu2022optimized}\textsuperscript{AAAI} & 70.43 & -- & 23.09 & 52.34 & 67.18 & 69.96 & 70.51 \\
& & QCFS\tnote{+}~\cite{DBLP:conf/iclr/BuFDDY022}\textsuperscript{ICLR} & 67.09 & 27.87 & 49.53 & 63.61 & 67.04 & 67.87 & 67.86 \\
& & Burst\tnote{*}~\cite{DBLP:conf/ijcai/Li022}\textsuperscript{IJCAI} & 80.69 & -- & -- & -- & 76.39 & 79.83 & 80.52 \\
\cmidrule{3-10}
& & \textbf{Ours (TPP) + QCFS} & 67.10 & \gc{46.88 (0.40)} & 64.77 (0.20) & 67.25 (0.12) & 67.74 (0.06) & 67.77 (0.05) & 67.79 (0.04) \\
\cmidrule{3-10}
& & \textbf{Ours (TPP)\tnote{*} + SNNC w/o Cali.} & 81.89 & 39.67 (0.99) & \gc{71.05 (0.68)} & \gc{78.97 (0.24)} & \gc{81.06 (0.05)} & \gc{81.61 (0.08)} & \gc{81.62 (0.05)} \\
\cmidrule{2-10}
& \multirow{12}{*}{VGG-16} 
& TSC\tnote{*}~\cite{han2020deep}\textsuperscript{ECCV} & 71.22 & -- & -- & -- & -- & -- & 69.86 \\
& & SNM\tnote{*}~\cite{wang2022signed}\textsuperscript{ICLR} & 74.13 & -- & -- & -- & 71.80 & 73.69 & 73.95 \\
& & SNNC-AP\tnote{*}~\cite{li2021free}\textsuperscript{ICML} & 77.89 & -- & -- & -- & 73.55 & 77.10 & \gc{77.86} \\
& & RTS\tnote{$\circ$}~\cite{DBLP:conf/iclr/DengG21}\textsuperscript{ICLR} & 76.13 & 23.76 & 43.81 & 56.23 & 67.61 & 73.45 & 75.23 \\
& & OPI\tnote{*}~\cite{bu2022optimized}\textsuperscript{AAAI} & 76.31 & -- & 60.49 & 70.72 & 74.82 & 75.97 & 76.25 \\
& & QCFS\tnote{+}~\cite{DBLP:conf/iclr/BuFDDY022}\textsuperscript{ICLR} & 76.21 & 69.29 & 73.89 & 75.98 & 76.53 & 76.54 & 76.60 \\
& & DDI~\cite{bojkovic2024data}\textsuperscript{AISTATS} & 70.44 & 51.21 & 53.65 & 57.12 & 61.61 & 70.44 & 73.82 \\
& & FTBC(+QCFS)~\cite{wu2024ftbc}\textsuperscript{ECCV} & 76.21 & 71.47 & 75.12 & 76.22 & 76.48 & 76.48 & 76.48 \\
\cmidrule{3-10}
& & \textbf{Ours (TPP) + RTS} & 76.13 & 37.88 (0.35) & 65.81 (0.27) & 73.05 (0.12) & 75.17 (0.17) & 75.64 (0.12) & 75.9 (0.08) \\
\cmidrule{3-10}
& & \textbf{Ours (TPP) + QCFS} & 76.21 & \gc{73.93 (0.22)} & \gc{76.03 (0.23)} & \gc{76.43 (0.07)} & 76.55 (0.03) & 76.55 (0.07) & 76.52 (0.04) \\
\cmidrule{3-10}
& & \textbf{Ours (TPP)\tnote{*} + SNNC w/o Cali.} & 77.87 & 59.23 (0.65) & 73.16 (0.17) & 76.05 (0.26) & \gc{77.16 (0.09)} & \gc{77.56 (0.13)} & 77.64 (0.04) \\
\bottomrule
\end{tabular}
\begin{tablenotes}
\item[] *: Without modification to ReLU of ANNs. +: Using authors' provided models and code. $\circ$: Self implemented.
\end{tablenotes}
\end{threeparttable}
}
\vspace{-15pt}
\end{table*}

\noindent \textbf{CIFAR dataset:} We further evaluate the performance of our methods on CIFAR-100 dataset and present the results in Table ~\ref{tab:ann-snn-imagenet-cifar}. We observe similar patterns as with the ImageNet. When comparing our method with ANN-SNN conversion methods which use non-\relu activations, e.g. QCFS and RTS, our method constantly outperforms RTS on ResNet-20 and VGG16. QCFS baseline suffers from necessity to train ANN models from scratch with custom activations, while our method is applicable to any ANN model with \relu-like activation. Furthermore, custom activation functions sometimes sacrifice the ANN performance as can be seen from the corresponding ANN accuracies. 

\noindent \textbf{CIFAR10-DVS dataset:} 
We evaluate our method on the event-based CIFAR10-DVS~\cite{li2017cifar10} dataset, comparing it with state-of-the-art direct training and ANN-SNN conversion methods (Table~\ref{tab:other-snn-training-methods}). Our approach demonstrates superior performance, achieving \gc{82.40\%} accuracy at just 8 timesteps and further improving to \gc{83.20\%} at 64 timesteps. Notably, our method outperforms the direct training method Spikformer~\cite{zhouSpikformerWhenSpiking2022} (80.90\%) and the ANN-SNN conversion method AdaFire~\cite{wang2024adaptive} (81.25\%).

\subsection{Comparison with other types of SNN training methods and models}


\begin{table}[!ht]
    \vspace{-5pt}
    \caption{Comparison of direct and hybrid training methods for SNNs on CIFAR-100, ImageNet and CIFAR10-DVS. Baseline results include the highest reported accuracy and corresponding latency.}
    \label{tab:other-snn-training-methods}
    \renewcommand\arraystretch{1.2}
    \centering
    \scalebox{0.65}{
    \begin{threeparttable}    
    \begin{tabular}{cccc}
        \toprule
        Method & Category & Timesteps & Accuracy \\
        \toprule
        \multicolumn{4}{c}{\textbf{VGG-16 [CIFAR-100]}} \\
        \midrule
        LM-H~\cite{hao2023progressive}\textsuperscript{ICLR} & Hybrid Training & 4 & 73.11 \\
        SEENN-II~\tnote{*}~\cite{DBLP:conf/nips/LiGKP23}\textsuperscript{NeurIPS} & Direct Training & 1.15\tnote{*} & 72.76 \\
        Dual-Phase~\cite{DBLP:journals/corr/abs-2205-07473} & Hybrid Training & 4 / 8 & 70.08 / 75.06 \\
        \textbf{Ours (TPP) + QCFS} & ANN-SNN & 4 / 8 & \gc{73.93 / 76.03} \\
        \midrule
        \multicolumn{4}{c}{\textbf{ResNet-20 [CIFAR-100]}} \\
        \midrule
        LM-H~\cite{hao2023progressive}\textsuperscript{ICLR} & Hybrid Training & 4 & 57.12 \\
        TTS~\cite{DBLP:conf/aaai/GuoCLPZHM24}\textsuperscript{AAAI} & Direct Training & 4 & 74.02 \\
        \textbf{Ours (TPP) + SNNC w/o Cali.} & ANN-SNN & 16 & \gc{78.97} \\             
        \midrule
        \multicolumn{4}{c}{\textbf{ResNet-34 [ImageNet]}} \\
        \midrule
        SEENN-I~\cite{DBLP:conf/nips/LiGKP23}\textsuperscript{NeurIPS} & Direct Training & 3.38~\tnote{*} & 64.66 \\
        RMP-Loss~\cite{DBLP:conf/iccv/GuoLCZPZHM23}\textsuperscript{ICCV} & Direct Training & 4 & 65.17 \\
        RecDis-SNN~\cite{guo2022recdis}\textsuperscript{CVPR} &  Direct Training & 6 & 67.33 \\
        SpikeConv~\cite{liu2022spikeconverter}\textsuperscript{AAAI} & Hybrid Training & 16 & 70.57 \\
        GAC-SNN~\cite{DBLP:conf/aaai/QiuZC0DL24}\textsuperscript{AAAI} & Direct Training & 6 & 70.42 \\
        TTS~\cite{DBLP:conf/aaai/GuoCLPZHM24}\textsuperscript{AAAI} & Direct Training & 4 & 70.74 \\
        SEENN-I~\cite{DBLP:conf/nips/LiGKP23}\textsuperscript{NeurIPS} & ANN-SNN & 29.53~\tnote{*} & 71.84 \\
        \textbf{Ours (TPP) + QCFS} & ANN-SNN & 16 & \gc{72.03} \\
        \textbf{Ours (TPP)+ SNNC w/o Cali.} & ANN-SNN & 32 & \gc{74.07} \\
        \midrule
        \multicolumn{4}{c}{\textbf{ResNet-18 [CIFAR10-DVS]}} \\
        \midrule 
        TA-SNN~\cite{yaoTemporalwiseAttentionSpiking2021}\textsuperscript{ICCV} & Direct Training & 10 & 72.00 \\
        PLIF~\cite{fangIncorporatingLearnableMembrane2021}\textsuperscript{ICCV} & Direct Training & 20 & 74.80 \\
        Dspkie~\cite{liDifferentiableSpikeRethinking2021}\textsuperscript{NeurIPS} & Direct Training & 10 & 75.40 \\
        DSR ~\cite{mengTrainingHighPerformanceLowLatency2022}\textsuperscript{CVPR} & Direct Training & 10 & 77.30 \\
        Spikformer ~\cite{zhouSpikformerWhenSpiking2022}\textsuperscript{ICLR} & Direct Training & 10 & 80.90 \\
        AdaFire ~\cite{wang2024adaptive}\textsuperscript{AAAI} & ANN-SNN & 8 & 81.25 \\
        \textbf{Ours (TPP) + SNNC w/o Cali.}  & ANN-SNN & \makecell{4 \\ \gc{8} \\ 16 \\ 32 \\ 64} & \makecell{77.90 \\ \gc{82.40} \\ 82.80 \\83.00 \\ 83.20} \\
        \bottomrule
    \end{tabular}
\begin{tablenotes}
\item[*] The average number of timesteps during inference on the test dataset.
\end{tablenotes}
\end{threeparttable}
}
\vspace{-10pt}
\end{table}
\raggedbottom 

We compare our approach with several state-of-the-art direct training and hybrid training methods as presented in Table~\ref{tab:other-snn-training-methods}. The comparison is founded on performance metrics like accuracy and the number of timesteps utilized during inference on the CIFAR-100 and ImageNet datasets. We benchmark our method against prominent approaches such as LM-H~\cite{hao2023progressive}, SEENN~\cite{DBLP:conf/nips/LiGKP23}, Dual-Phase~\cite{DBLP:journals/corr/abs-2205-07473}, TTS~\cite{DBLP:conf/aaai/GuoCLPZHM24}, RMP-Loss~\cite{DBLP:conf/iccv/GuoLCZPZHM23}, RecDis-SNN~\cite{guo2022recdis}, SpikeConv~\cite{liu2022spikeconverter}, and GAC-SNN~\cite{DBLP:conf/aaai/QiuZC0DL24}. We showcase the best accuracy comparable to state-of-the-art methods achieved by our approach with minimal timesteps. We prioritize accuracy, but direct training and hybrid training opt for a lower number of timesteps and sacrifice accuracy. We outperform LM-H~\cite{hao2023progressive} and Dual-Phase~\cite{DBLP:journals/corr/abs-2205-07473} for VGG-16 on CIFAR-100. For ResNet-20 on CIFAR-100, we have higher accuracy but longer timesteps. Additionally, for ResNet-34 on the ImageNet dataset, the accuracy of our method with QCFS with 16 timesteps is higher than that of SpikeConv~\cite{liu2022spikeconverter} with the same number of timesteps. We also achieve higher accuracy with longer timesteps as expected. Overall, our approach demonstrates promising performance and competitiveness in comparison with the existing SNN training methods.




\begin{figure*}[!ht]
\centering
\subfigure[Layer 4]{
    \includegraphics[width=0.32\linewidth]{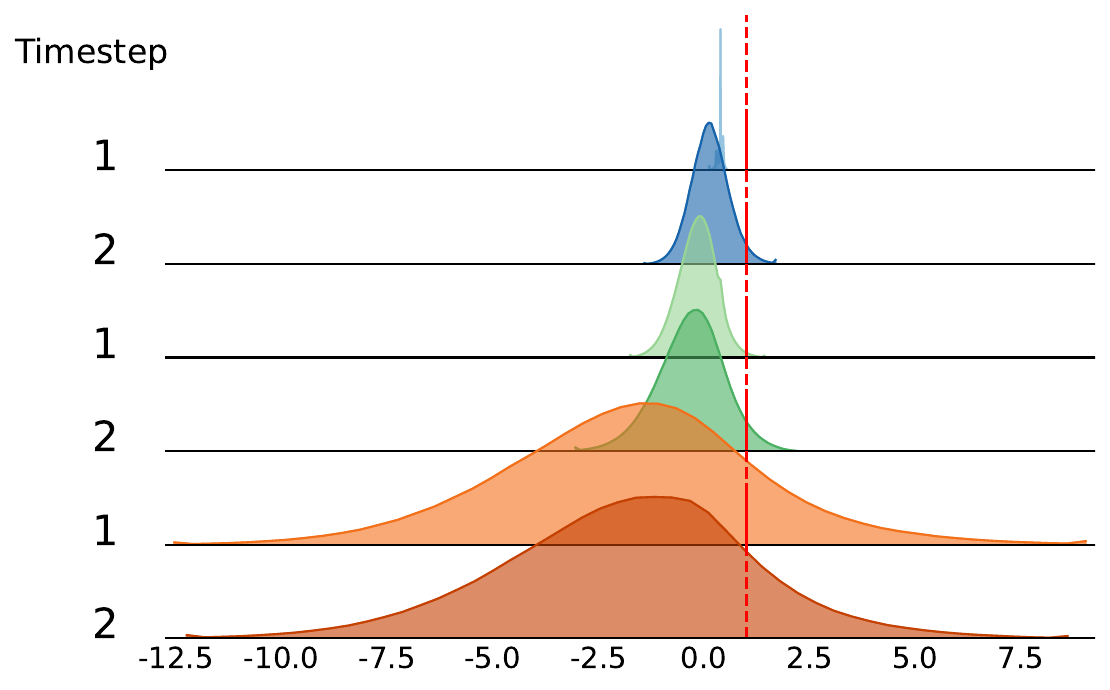}}
\subfigure[Layer 8]{
    \includegraphics[width=0.32\linewidth]{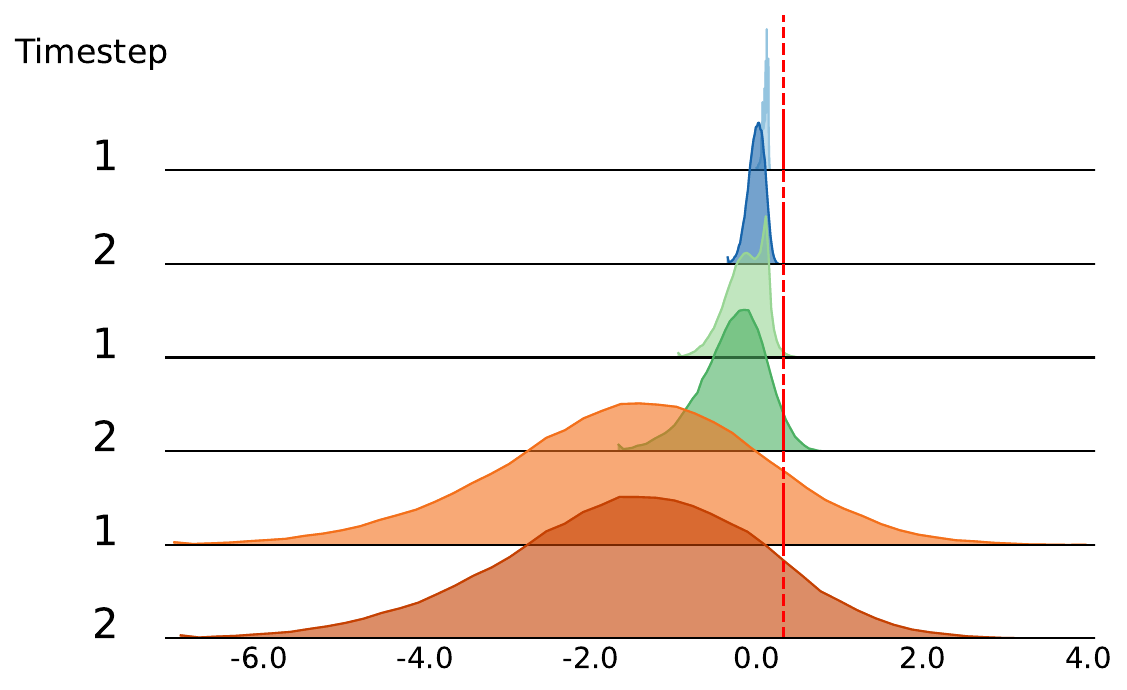}}
\subfigure[Layer 12]{
    \includegraphics[width=0.32\linewidth]{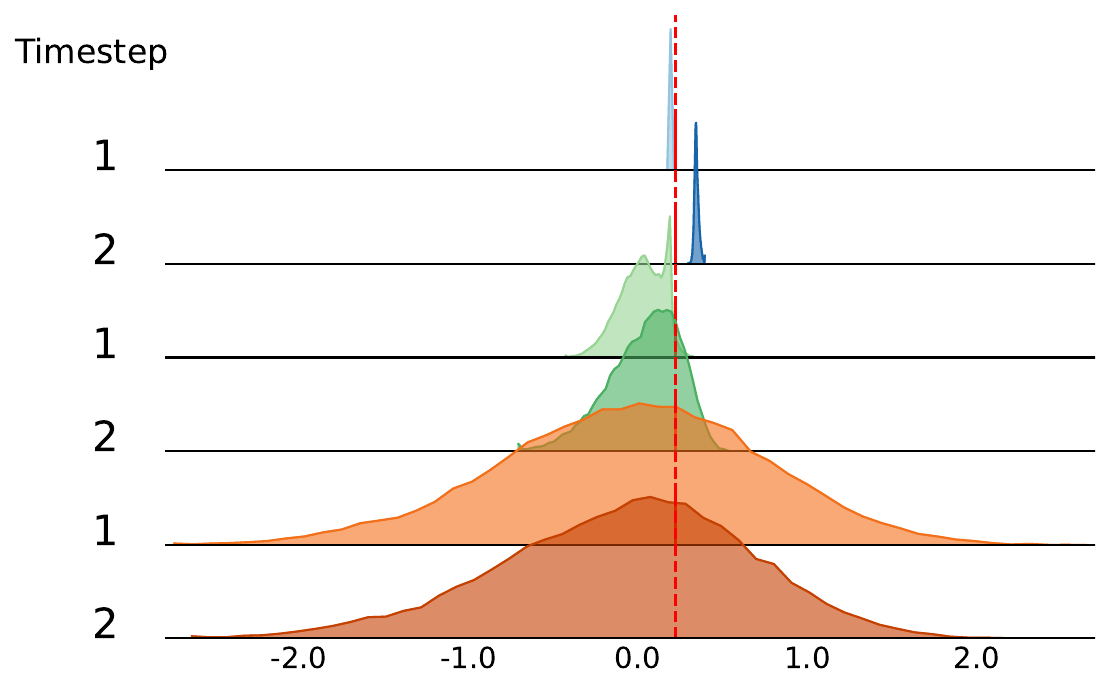}}    
\vskip -0.1in
\caption{The membrane potential distributions of the first channel (randomly selected) across three modes (baseline, shuffle, and probabilistic) in VGG-16 on CIFAR-100. For comparative analysis, the first two timesteps ($t=1$, $t=2$) from a total of eight timesteps ($T=8$) are selected for each mode. The baseline mode (blue) attains an accuracy of 24.22\%, while the shuffle mode (light green) enhances accuracy to 70.54\%, and the probabilistic mode (dark orange) further improves accuracy to 73.42\%. The distributions are presented prior to neuronal firing, with the red dashed line indicating the threshold voltage (Vth) for the respective layer (see Appendix~\ref{appendix:mem-pot-dist}).}
 \label{fig:mem-pot-distribution}
 \vspace{-5pt}
\end{figure*}

\subsection{Spiking activity}
The event driven nature of various neuromorphic chips implies that the energy consumption is directly proportional to the spiking activity, i.e., the number of spikes produced throughout the network: the energy is consumed in the presence of spikes. To this end, we tested our proposed method (TPP) for the spike activity and compared with the baselines. For a given model, we counted the average number of spikes produced after each layer, per sample, for both the baseline and our method. Figure~\ref{fig:spike-counts} shows the example of RTS and RTS + TPP.  Both the baseline and our method exhibit similar spike counts. In particular, our method constantly outperforms the baselines, and possibly in doing so it needs longer average latency per sample ($T$). However, the energy consumed is approximately the same as that for the baseline in time $T$. The complete tables are present in Appendix \ref{appendix:firing-counts}, where we provide more detailed picture of spike activities.
\begin{figure}[h]
\centering
\includegraphics[width=0.9\linewidth]{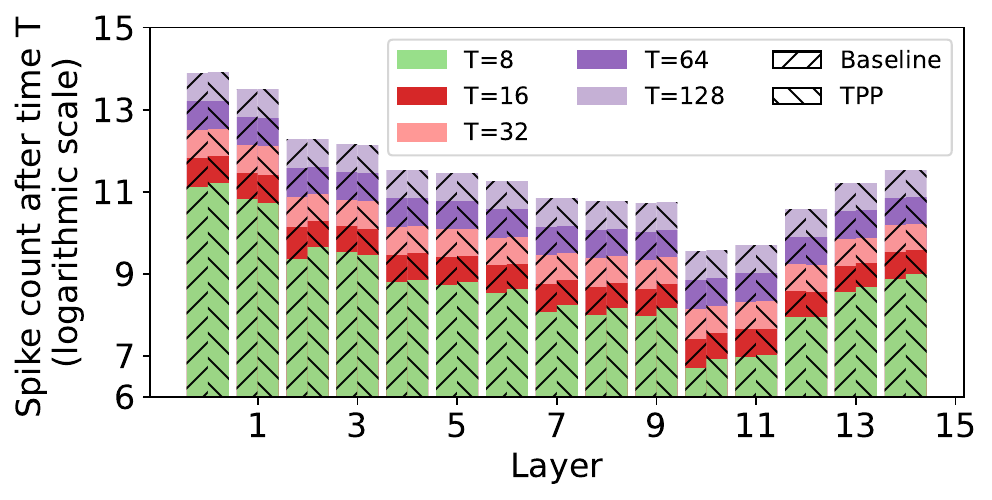}
\vskip -0.2in
\caption{\small Spike counts of VGG-16 on CIFAR-100 of RTS baseline compared with RTS+TPP. Note: The bar height from bottom indicates the spike counts after each timestep T, and the color of longer Ts is overlaid by shorter Ts (see Appendix \ref{appendix:firing-counts}) }
\label{fig:spike-counts}
\vspace{-10pt}
\end{figure}

\subsection{Membrane potential distribution in early time steps} 

In Figure \ref{fig:mem-pot-distribution} we compare the membrane potential distributions for baseline models and two methods that we studied in the paper, the permutations applied on spike trains and TPP method. Once again, it can be seen another reason for performance degradation of baseline models in low latency, as the membrane potential is not variable enough to produce spike informative spike trains, which is particularly visible in deeper layers. On the other side ``permuted'' and TPP models show sufficient variability throughout the layers. 

By increasing the latency, the baseline models can recover some of the variability and spike production, as can be seen in Figure \ref{fig:spike-counts}. But, due to the misplacement of spike trains through temporal dimension, they are still not able to pick up on the ANN performance.

\section{Conclusions and future work} 
\label{sec:conclusion}

This work identified the phenomenon of ``temporal misalignment'' in ANN-SNN conversion, where random spike rearrangement enhances performance. We introduced two-phase probabilistic (TPP) spiking neurons, designed to intrinsically perform the effect of spike permutations. We show biological plausibility of such neurons as well as the hardware friendlines of the underlying mechanisms. We demonstrate their effectiveness through exhaustive experiments on large scale datasets, showing their competing performance compared to SOTA ANN-SNN conversion and direct training methods. 

In the future work, we aim to study the effect of permutations and probabilistic spiking in combination with directly trained SNN models.

\section*{Impact Statements}

This paper presents work whose goal is to advance the field of Machine Learning. There are many potential societal consequences of our work, none which we feel must be specifically highlighted here.

\bibliography{main}

\begin{thebibliography}{102}
\providecommand{\natexlab}[1]{#1}
\providecommand{\url}[1]{\texttt{#1}}
\expandafter\ifx\csname urlstyle\endcsname\relax
  \providecommand{\doi}[1]{doi: #1}\else
  \providecommand{\doi}{doi: \begingroup \urlstyle{rm}\Url}\fi

\bibitem[loi()]{loihi2}
Taking {N}euromorphic {C}omputing to the {N}ext {L}evel with {L}oihi 2.
\newblock \url{https://download.intel.com/newsroom/2021/new-technologies/neuromorphic-computing-loihi-2-brief.pdf}.
\newblock Accessed: 16-05-2023.

\bibitem[Bal \& Sengupta(2024)Bal and Sengupta]{bal2023spikingbert}
Bal, M. and Sengupta, A.
\newblock Spikingbert: Distilling bert to train spiking language models using implicit differentiation.
\newblock In \emph{Proceedings of the AAAI conference on artificial intelligence}, 2024.

\bibitem[Bellec et~al.(2018)Bellec, Salaj, Subramoney, Legenstein, and Maass]{bellec2018long}
Bellec, G., Salaj, D., Subramoney, A., Legenstein, R., and Maass, W.
\newblock Long short-term memory and learning-to-learn in networks of spiking neurons.
\newblock 2018.

\bibitem[Bojkovic et~al.(2024)Bojkovic, Anumasa, De~Masi, Gu, and Xiong]{bojkovic2024data}
Bojkovic, V., Anumasa, S., De~Masi, G., Gu, B., and Xiong, H.
\newblock Data driven threshold and potential initialization for spiking neural networks.
\newblock In \emph{International Conference on Artificial Intelligence and Statistics}, 2024.

\bibitem[Braspenning et~al.(1995)Braspenning, Thuijsman, and Weijters]{braspenning1995artificial}
Braspenning, P.~J., Thuijsman, F., and Weijters, A. J. M.~M.
\newblock \emph{Artificial neural networks: an introduction to ANN theory and practice}, volume 931.
\newblock Springer Science \& Business Media, 1995.

\bibitem[Bu et~al.(2022{\natexlab{a}})Bu, Ding, Yu, and Huang]{bu2022optimized}
Bu, T., Ding, J., Yu, Z., and Huang, T.
\newblock Optimized potential initialization for low-latency spiking neural networks.
\newblock In \emph{Proceedings of the AAAI Conference on Artificial Intelligence}, 2022{\natexlab{a}}.

\bibitem[Bu et~al.(2022{\natexlab{b}})Bu, Fang, Ding, Dai, Yu, and Huang]{DBLP:conf/iclr/BuFDDY022}
Bu, T., Fang, W., Ding, J., Dai, P., Yu, Z., and Huang, T.
\newblock Optimal {ANN-SNN} conversion for high-accuracy and ultra-low-latency spiking neural networks.
\newblock In \emph{The Tenth International Conference on Learning Representations, {ICLR} 2022, Virtual Event, April 25-29, 2022}, 2022{\natexlab{b}}.

\bibitem[Bu et~al.(2022{\natexlab{c}})Bu, Fang, Ding, Dai, Yu, and Huang]{bu2022optimal}
Bu, T., Fang, W., Ding, J., Dai, P., Yu, Z., and Huang, T.
\newblock Optimal {ANN-SNN} conversion for high-accuracy and ultra-low-latency spiking neural networks.
\newblock In \emph{International Conference on Learning Representations}, 2022{\natexlab{c}}.

\bibitem[Cao et~al.(2015)Cao, Chen, and Khosla]{cao2015spiking}
Cao, Y., Chen, Y., and Khosla, D.
\newblock Spiking deep convolutional neural networks for energy-efficient object recognition.
\newblock \emph{International Journal of Computer Vision}, 113\penalty0 (1):\penalty0 54--66, 2015.

\bibitem[Cheng et~al.(2020)Cheng, Hao, Xu, and Xu]{cheng2020lisnn}
Cheng, X., Hao, Y., Xu, J., and Xu, B.
\newblock Lisnn: Improving spiking neural networks with lateral interactions for robust object recognition.
\newblock In \emph{IJCAI}, pp.\  1519--1525, 2020.

\bibitem[Connors \& Gutnick(1990)Connors and Gutnick]{connors1990intrinsic}
Connors, B.~W. and Gutnick, M.~J.
\newblock Intrinsic firing patterns of diverse neocortical neurons.
\newblock \emph{Trends in neurosciences}, 13\penalty0 (3):\penalty0 99--104, 1990.

\bibitem[Das(2023)]{das2023design}
Das, A.
\newblock A design flow for scheduling spiking deep convolutional neural networks on heterogeneous neuromorphic system-on-chip.
\newblock \emph{ACM Transactions on Embedded Computing Systems}, 2023.

\bibitem[Davies et~al.(2018)Davies, Srinivasa, Lin, Chinya, Cao, Choday, Dimou, Joshi, Imam, Jain, et~al.]{davies2018loihi}
Davies, M., Srinivasa, N., Lin, T.-H., Chinya, G., Cao, Y., Choday, S.~H., Dimou, G., Joshi, P., Imam, N., Jain, S., et~al.
\newblock Loihi: A neuromorphic manycore processor with on-chip learning.
\newblock \emph{Ieee Micro}, 38\penalty0 (1):\penalty0 82--99, 2018.

\bibitem[DeBole et~al.(2019)DeBole, Taba, Amir, Akopyan, Andreopoulos, Risk, Kusnitz, Otero, Nayak, Appuswamy, et~al.]{debole2019truenorth}
DeBole, M.~V., Taba, B., Amir, A., Akopyan, F., Andreopoulos, A., Risk, W.~P., Kusnitz, J., Otero, C.~O., Nayak, T.~K., Appuswamy, R., et~al.
\newblock Truenorth: Accelerating from zero to 64 million neurons in 10 years.
\newblock \emph{Computer}, 52\penalty0 (5):\penalty0 20--29, 2019.

\bibitem[Deng et~al.(2009)Deng, Dong, Socher, Li, Li, and Fei-Fei]{deng2009imagenet}
Deng, J., Dong, W., Socher, R., Li, L.-J., Li, K., and Fei-Fei, L.
\newblock Imagenet: A large-scale hierarchical image database.
\newblock In \emph{2009 IEEE Conference on Computer Vision and Pattern Recognition}, pp.\  248--255, 2009.
\newblock \doi{10.1109/CVPR.2009.5206848}.

\bibitem[Deng \& Gu(2021{\natexlab{a}})Deng and Gu]{DBLP:conf/iclr/DengG21}
Deng, S. and Gu, S.
\newblock Optimal conversion of conventional artificial neural networks to spiking neural networks.
\newblock In \emph{9th International Conference on Learning Representations, {ICLR} 2021, Virtual Event, Austria, May 3-7, 2021}, 2021{\natexlab{a}}.

\bibitem[Deng \& Gu(2021{\natexlab{b}})Deng and Gu]{deng2021optimal}
Deng, S. and Gu, S.
\newblock Optimal conversion of conventional artificial neural networks to spiking neural networks.
\newblock \emph{International Conference on Learning Representations}, 2021{\natexlab{b}}.

\bibitem[Diehl \& Cook(2015)Diehl and Cook]{diehl2015unsupervised}
Diehl, P.~U. and Cook, M.
\newblock Unsupervised learning of digit recognition using spike-timing-dependent plasticity.
\newblock \emph{Frontiers in computational neuroscience}, 9:\penalty0 99, 2015.

\bibitem[Diehl et~al.(2015)Diehl, Neil, Binas, Cook, Liu, and Pfeiffer]{diehl2015fast}
Diehl, P.~U., Neil, D., Binas, J., Cook, M., Liu, S.-C., and Pfeiffer, M.
\newblock Fast-classifying, high-accuracy spiking deep networks through weight and threshold balancing.
\newblock In \emph{2015 International Joint Conference on Neural Networks (IJCNN)}, pp.\  1--8. ieee, 2015.

\bibitem[Ding et~al.(2021)Ding, Yu, Tian, and Huang]{ding2021optimal}
Ding, J., Yu, Z., Tian, Y., and Huang, T.
\newblock Optimal ann-snn conversion for fast and accurate inference in deep spiking neural networks.
\newblock \emph{In International Joint Conference on Artificial Intelligence}, pp.\  2328--2336, 2021.

\bibitem[Faisal et~al.(2008)Faisal, Selen, and Wolpert]{faisal2008noise}
Faisal, A.~A., Selen, L.~P., and Wolpert, D.~M.
\newblock Noise in the nervous system.
\newblock \emph{Nature reviews neuroscience}, 9\penalty0 (4):\penalty0 292--303, 2008.

\bibitem[Fang et~al.(2021{\natexlab{a}})Fang, Yu, Chen, Huang, Masquelier, and Tian]{fang2021deep}
Fang, W., Yu, Z., Chen, Y., Huang, T., Masquelier, T., and Tian, Y.
\newblock Deep residual learning in spiking neural networks.
\newblock \emph{Advances in Neural Information Processing Systems}, 34:\penalty0 21056--21069, 2021{\natexlab{a}}.

\bibitem[Fang et~al.(2021{\natexlab{b}})Fang, Yu, Chen, Masquelier, Huang, and Tian]{fang2021incorporating}
Fang, W., Yu, Z., Chen, Y., Masquelier, T., Huang, T., and Tian, Y.
\newblock Incorporating learnable membrane time constant to enhance learning of spiking neural networks.
\newblock In \emph{Proceedings of the IEEE/CVF international conference on computer vision}, 2021{\natexlab{b}}.

\bibitem[Fang et~al.(2021{\natexlab{c}})Fang, Yu, Chen, Masquelier, Huang, and Tian]{fangIncorporatingLearnableMembrane2021}
Fang, W., Yu, Z., Chen, Y., Masquelier, T., Huang, T., and Tian, Y.
\newblock Incorporating learnable membrane time constant to enhance learning of spiking neural networks.
\newblock In \emph{Proceedings of the {{IEEE}}/{{CVF International Conference}} on {{Computer Vision}}}, pp.\  2661--2671, 2021{\natexlab{c}}.

\bibitem[Furber et~al.(2014)Furber, Galluppi, Temple, and Plana]{furber2014spinnaker}
Furber, S.~B., Galluppi, F., Temple, S., and Plana, L.~A.
\newblock The spinnaker project.
\newblock \emph{Proceedings of the IEEE}, 102\penalty0 (5):\penalty0 652--665, 2014.

\bibitem[Gonzalez et~al.(2024)Gonzalez, Huang, Kelber, Nazeer, Langer, Liu, Lohrmann, Rostami, Sch{\"o}ne, Vogginger, et~al.]{gonzalez2024spinnaker2}
Gonzalez, H.~A., Huang, J., Kelber, F., Nazeer, K.~K., Langer, T., Liu, C., Lohrmann, M., Rostami, A., Sch{\"o}ne, M., Vogginger, B., et~al.
\newblock Spinnaker2: A large-scale neuromorphic system for event-based and asynchronous machine learning.
\newblock \emph{arXiv preprint arXiv:2401.04491}, 2024.

\bibitem[Guo et~al.(2022)Guo, Tong, Chen, Zhang, Liu, Ma, and Huang]{guo2022recdis}
Guo, Y., Tong, X., Chen, Y., Zhang, L., Liu, X., Ma, Z., and Huang, X.
\newblock Recdis-snn: Rectifying membrane potential distribution for directly training spiking neural networks.
\newblock In \emph{Proceedings of the IEEE/CVF Conference on Computer Vision and Pattern Recognition}, 2022.

\bibitem[Guo et~al.(2023)Guo, Liu, Chen, Zhang, Peng, Zhang, Huang, and Ma]{DBLP:conf/iccv/GuoLCZPZHM23}
Guo, Y., Liu, X., Chen, Y., Zhang, L., Peng, W., Zhang, Y., Huang, X., and Ma, Z.
\newblock Rmp-loss: Regularizing membrane potential distribution for spiking neural networks.
\newblock In \emph{{IEEE/CVF} International Conference on Computer Vision, {ICCV} 2023, Paris, France, October 1-6, 2023}, 2023.

\bibitem[Guo et~al.(2024)Guo, Chen, Liu, Peng, Zhang, Huang, and Ma]{DBLP:conf/aaai/GuoCLPZHM24}
Guo, Y., Chen, Y., Liu, X., Peng, W., Zhang, Y., Huang, X., and Ma, Z.
\newblock Ternary spike: Learning ternary spikes for spiking neural networks.
\newblock In \emph{Thirty-Eighth {AAAI} Conference on Artificial Intelligence, {AAAI} 2024, Thirty-Sixth Conference on Innovative Applications of Artificial Intelligence, {IAAI} 2024, Fourteenth Symposium on Educational Advances in Artificial Intelligence, {EAAI} 2014, February 20-27, 2024, Vancouver, Canada}, 2024.

\bibitem[Han \& Roy(2020)Han and Roy]{han2020deep}
Han, B. and Roy, K.
\newblock Deep spiking neural network: Energy efficiency through time based coding.
\newblock In \emph{European Conference on Computer Vision}. Springer, 2020.

\bibitem[Han et~al.(2020)Han, Srinivasan, and Roy]{han2020rmp}
Han, B., Srinivasan, G., and Roy, K.
\newblock {RMP-SNN}: Residual membrane potential neuron for enabling deeper high-accuracy and low-latency spiking neural network.
\newblock In \emph{Proceedings of the IEEE/CVF conference on Computer Vision and Pattern Recognition (CVPR)}, pp.\  13558--13567, 2020.

\bibitem[Hao et~al.(2023{\natexlab{a}})Hao, Bu, Ding, Huang, and Yu]{DBLP:conf/aaai/HaoBD0Y23}
Hao, Z., Bu, T., Ding, J., Huang, T., and Yu, Z.
\newblock Reducing {ANN-SNN} conversion error through residual membrane potential.
\newblock In \emph{Thirty-Seventh {AAAI} Conference on Artificial Intelligence, {AAAI} 2023, Thirty-Fifth Conference on Innovative Applications of Artificial Intelligence, {IAAI} 2023, Thirteenth Symposium on Educational Advances in Artificial Intelligence, {EAAI} 2023, Washington, DC, USA, February 7-14, 2023}, 2023{\natexlab{a}}.

\bibitem[Hao et~al.(2023{\natexlab{b}})Hao, Shi, Huang, Bu, Yu, and Huang]{hao2023progressive}
Hao, Z., Shi, X., Huang, Z., Bu, T., Yu, Z., and Huang, T.
\newblock A progressive training framework for spiking neural networks with learnable multi-hierarchical model.
\newblock In \emph{The Twelfth International Conference on Learning Representations}, 2023{\natexlab{b}}.

\bibitem[He et~al.(2016)He, Zhang, Ren, and Sun]{he2016residual}
He, K., Zhang, X., Ren, S., and Sun, J.
\newblock Deep residual learning for image recognition.
\newblock In \emph{2016 IEEE Conference on Computer Vision and Pattern Recognition (CVPR)}, pp.\  770--778, 2016.
\newblock \doi{10.1109/CVPR.2016.90}.

\bibitem[Ho \& Chang(2021)Ho and Chang]{ho2021tcl}
Ho, N.-D. and Chang, I.-J.
\newblock Tcl: an ann-to-snn conversion with trainable clipping layers.
\newblock In \emph{2021 58th ACM/IEEE Design Automation Conference (DAC)}. IEEE, 2021.

\bibitem[Hodgkin \& Huxley(1952)Hodgkin and Huxley]{hodgkin1952quantitative}
Hodgkin, A.~L. and Huxley, A.~F.
\newblock A quantitative description of membrane current and its application to conduction and excitation in nerve.
\newblock \emph{The Journal of physiology}, 117\penalty0 (4):\penalty0 500, 1952.

\bibitem[Hu et~al.(2023)Hu, Zheng, Jiang, and Pan]{hu2023fast}
Hu, Y., Zheng, Q., Jiang, X., and Pan, G.
\newblock Fast-snn: fast spiking neural network by converting quantized ann.
\newblock \emph{IEEE Transactions on Pattern Analysis and Machine Intelligence}, 2023.

\bibitem[Izhikevich(2003)]{izhikevich2003simple}
Izhikevich, E.~M.
\newblock Simple model of spiking neurons.
\newblock \emph{IEEE Transactions on neural networks}, 14\penalty0 (6):\penalty0 1569--1572, 2003.

\bibitem[Izhikevich(2007)]{izhikevich2007dynamical}
Izhikevich, E.~M.
\newblock \emph{Dynamical systems in neuroscience}.
\newblock MIT press, 2007.

\bibitem[Kim et~al.(2020{\natexlab{a}})Kim, Kim, and Kim]{DBLP:conf/nips/KimKK20a}
Kim, J., Kim, K., and Kim, J.
\newblock Unifying activation- and timing-based learning rules for spiking neural networks.
\newblock In \emph{Advances in Neural Information Processing Systems 33: Annual Conference on Neural Information Processing Systems 2020, NeurIPS 2020, December 6-12, 2020, virtual}, 2020{\natexlab{a}}.

\bibitem[Kim et~al.(2020{\natexlab{b}})Kim, Park, Na, and Yoon]{kim2020spiking}
Kim, S., Park, S., Na, B., and Yoon, S.
\newblock Spiking-yolo: spiking neural network for energy-efficient object detection.
\newblock In \emph{Proceedings of the AAAI conference on artificial intelligence}, volume~34, pp.\  11270--11277, 2020{\natexlab{b}}.

\bibitem[Krahe \& Gabbiani(2004)Krahe and Gabbiani]{krahe2004burst}
Krahe, R. and Gabbiani, F.
\newblock Burst firing in sensory systems.
\newblock \emph{Nature Reviews Neuroscience}, 5\penalty0 (1):\penalty0 13--23, 2004.

\bibitem[Krizhevsky \& Hinton(2009)Krizhevsky and Hinton]{krizhevsky2009learning}
Krizhevsky, A. and Hinton, G.
\newblock Learning multiple layers of features from tiny images.
\newblock \emph{https://www.cs.toronto.edu/~kriz/cifar.html}, 2009.

\bibitem[Krizhevsky et~al.(2010)Krizhevsky, Nair, and Hinton]{krizhevsky2010cifar}
Krizhevsky, A., Nair, V., and Hinton, G.
\newblock Cifar-10 (canadian institute for advanced research).
\newblock \emph{URL http://www. cs. toronto. edu/kriz/cifar. html}, 2010.

\bibitem[LeCun et~al.(1998)LeCun, Bottou, Bengio, and Haffner]{lecun1998gradient}
LeCun, Y., Bottou, L., Bengio, Y., and Haffner, P.
\newblock Gradient-based learning applied to document recognition.
\newblock \emph{Proceedings of the IEEE}, 86\penalty0 (11):\penalty0 2278--2324, 1998.

\bibitem[Li et~al.(2017)Li, Liu, Ji, Li, and Shi]{li2017cifar10}
Li, H., Liu, H., Ji, X., Li, G., and Shi, L.
\newblock Cifar10-dvs: an event-stream dataset for object classification.
\newblock \emph{Frontiers in neuroscience}, 11:\penalty0 309, 2017.

\bibitem[Li \& Zeng(2022)Li and Zeng]{DBLP:conf/ijcai/Li022}
Li, Y. and Zeng, Y.
\newblock Efficient and accurate conversion of spiking neural network with burst spikes.
\newblock In \emph{Proceedings of the Thirty-First International Joint Conference on Artificial Intelligence, {IJCAI} 2022, Vienna, Austria, 23-29 July 2022}, 2022.

\bibitem[Li et~al.(2021{\natexlab{a}})Li, Deng, Dong, Gong, and Gu]{li2021free}
Li, Y., Deng, S., Dong, X., Gong, R., and Gu, S.
\newblock A free lunch from ann: Towards efficient, accurate spiking neural networks calibration.
\newblock In \emph{International Conference on Machine Learning}, pp.\  6316--6325. PMLR, 2021{\natexlab{a}}.

\bibitem[Li et~al.(2021{\natexlab{b}})Li, Guo, Zhang, Deng, Hai, and Gu]{liDifferentiableSpikeRethinking2021}
Li, Y., Guo, Y., Zhang, S., Deng, S., Hai, Y., and Gu, S.
\newblock Differentiable spike: {{Rethinking}} gradient-descent for training spiking neural networks.
\newblock \emph{Advances in Neural Information Processing Systems}, 34:\penalty0 23426--23439, 2021{\natexlab{b}}.

\bibitem[Li et~al.(2022)Li, Zhao, and Zeng]{li2022bsnn}
Li, Y., Zhao, D., and Zeng, Y.
\newblock Bsnn: Towards faster and better conversion of artificial neural networks to spiking neural networks with bistable neurons.
\newblock \emph{Frontiers in Neuroscience}, 16, 2022.
\newblock ISSN 1662-453X.
\newblock \doi{10.3389/fnins.2022.991851}.
\newblock URL \url{https://www.frontiersin.org/articles/10.3389/fnins.2022.991851}.

\bibitem[Li et~al.(2023)Li, Geller, Kim, and Panda]{DBLP:conf/nips/LiGKP23}
Li, Y., Geller, T., Kim, Y., and Panda, P.
\newblock {SEENN:} towards temporal spiking early exit neural networks.
\newblock In \emph{Advances in Neural Information Processing Systems 36: Annual Conference on Neural Information Processing Systems 2023, NeurIPS 2023, New Orleans, LA, USA, December 10 - 16, 2023}, 2023.

\bibitem[Liu et~al.(2022)Liu, Zhao, Chen, Wang, and Jiang]{liu2022spikeconverter}
Liu, F., Zhao, W., Chen, Y., Wang, Z., and Jiang, L.
\newblock Spikeconverter: An efficient conversion framework zipping the gap between artificial neural networks and spiking neural networks.
\newblock In \emph{Proceedings of the AAAI Conference on Artificial Intelligence}, volume~36, pp.\  1692--1701, 2022.

\bibitem[Llin{\'a}s \& Jahnsen(1982)Llin{\'a}s and Jahnsen]{llinas1982electrophysiology}
Llin{\'a}s, R. and Jahnsen, H.
\newblock Electrophysiology of mammalian thalamic neurones in vitro.
\newblock \emph{Nature}, 297\penalty0 (5865):\penalty0 406--408, 1982.

\bibitem[Ma et~al.(2023)Ma, Jin, Sun, Li, Wu, Hu, Yang, Tang, Zhu, Lin, and Pan]{DBLP:journals/corr/abs-2312-17582}
Ma, D., Jin, X., Sun, S., Li, Y., Wu, X., Hu, Y., Yang, F., Tang, H., Zhu, X., Lin, P., and Pan, G.
\newblock Darwin3: {A} large-scale neuromorphic chip with a novel {ISA} and on-chip learning.
\newblock \emph{CoRR}, 2023.

\bibitem[Maass(1997)]{maas1997third}
Maass, W.
\newblock Networks of spiking neurons: The third generation of neural network models.
\newblock \emph{Neural Networks}, 10\penalty0 (9):\penalty0 1659--1671, 1997.
\newblock ISSN 0893-6080.
\newblock \doi{https://doi.org/10.1016/S0893-6080(97)00011-7}.
\newblock URL \url{https://www.sciencedirect.com/science/article/pii/S0893608097000117}.

\bibitem[Maass \& Natschl{\"a}ger(1997)Maass and Natschl{\"a}ger]{maass1997networks}
Maass, W. and Natschl{\"a}ger, T.
\newblock Networks of spiking neurons can emulate arbitrary hopfield nets in temporal coding.
\newblock \emph{Network: Computation in Neural Systems}, 8\penalty0 (4):\penalty0 355--371, 1997.

\bibitem[maintainers \& contributors(2016)maintainers and contributors]{torchvision2016}
maintainers, T. and contributors.
\newblock Torchvision: Pytorch's computer vision library.
\newblock \url{https://github.com/pytorch/vision}, 2016.

\bibitem[McCulloch \& Pitts(1943)McCulloch and Pitts]{mcculloch1943logical}
McCulloch, W.~S. and Pitts, W.
\newblock A logical calculus of the ideas immanent in nervous activity.
\newblock \emph{The bulletin of mathematical biophysics}, 5\penalty0 (4):\penalty0 115--133, 1943.

\bibitem[Meng et~al.(2022)Meng, Xiao, Yan, Wang, Lin, and Luo]{mengTrainingHighPerformanceLowLatency2022}
Meng, Q., Xiao, M., Yan, S., Wang, Y., Lin, Z., and Luo, Z.-Q.
\newblock Training {{High-Performance Low-Latency Spiking Neural Networks}} by {{Differentiation}} on {{Spike Representation}}.
\newblock In \emph{Proceedings of the {{IEEE}}/{{CVF Conference}} on {{Computer Vision}} and {{Pattern Recognition}}}, pp.\  12444--12453, 2022.

\bibitem[Merolla et~al.(2014)Merolla, Arthur, Alvarez-Icaza, Cassidy, Sawada, Akopyan, Jackson, Imam, Guo, Nakamura, et~al.]{merolla2014million}
Merolla, P.~A., Arthur, J.~V., Alvarez-Icaza, R., Cassidy, A.~S., Sawada, J., Akopyan, F., Jackson, B.~L., Imam, N., Guo, C., Nakamura, Y., et~al.
\newblock A million spiking-neuron integrated circuit with a scalable communication network and interface.
\newblock \emph{Science}, 345\penalty0 (6197):\penalty0 668--673, 2014.

\bibitem[Mukhoty et~al.(2024)Mukhoty, Bojkovi{\'c}, de~Vazelhes, Zhao, De~Masi, Xiong, and Gu]{mukhoty2024direct}
Mukhoty, B., Bojkovi{\'c}, V., de~Vazelhes, W., Zhao, X., De~Masi, G., Xiong, H., and Gu, B.
\newblock Direct training of snn using local zeroth order method.
\newblock \emph{Advances in Neural Information Processing Systems}, 36, 2024.

\bibitem[Neftci et~al.(2019)Neftci, Mostafa, and Zenke]{neftci2019surrogate}
Neftci, E.~O., Mostafa, H., and Zenke, F.
\newblock Surrogate gradient learning in spiking neural networks: Bringing the power of gradient-based optimization to spiking neural networks.
\newblock \emph{IEEE Signal Processing Magazine}, 36\penalty0 (6):\penalty0 51--63, 2019.

\bibitem[O'Connor et~al.(2018)O'Connor, Gavves, Reisser, and Welling]{DBLP:conf/iclr/OConnorGRW18}
O'Connor, P., Gavves, E., Reisser, M., and Welling, M.
\newblock Temporally efficient deep learning with spikes.
\newblock In \emph{6th International Conference on Learning Representations, {ICLR} 2018, Vancouver, BC, Canada, April 30 - May 3, 2018, Conference Track Proceedings}, 2018.

\bibitem[Pagliarini et~al.(2019)Pagliarini, Bhuin, Isgenc, Biswas, and Pileggi]{pagliarini2019probabilistic}
Pagliarini, S.~N., Bhuin, S., Isgenc, M.~M., Biswas, A.~K., and Pileggi, L.
\newblock A probabilistic synapse with strained mtjs for spiking neural networks.
\newblock \emph{IEEE Transactions on Neural Networks and Learning Systems}, 31\penalty0 (4):\penalty0 1113--1123, 2019.

\bibitem[Paszke et~al.(2019)Paszke, Gross, Massa, Lerer, Bradbury, Chanan, Killeen, Lin, Gimelshein, Antiga, et~al.]{paszke2019pytorch}
Paszke, A., Gross, S., Massa, F., Lerer, A., Bradbury, J., Chanan, G., Killeen, T., Lin, Z., Gimelshein, N., Antiga, L., et~al.
\newblock Pytorch: An imperative style, high-performance deep learning library.
\newblock \emph{Advances in neural information processing systems}, 32, 2019.

\bibitem[Pehle et~al.(2022)Pehle, Billaudelle, Cramer, Kaiser, Schreiber, Stradmann, Weis, Leibfried, M{\"u}ller, and Schemmel]{pehle2022brainscales}
Pehle, C., Billaudelle, S., Cramer, B., Kaiser, J., Schreiber, K., Stradmann, Y., Weis, J., Leibfried, A., M{\"u}ller, E., and Schemmel, J.
\newblock The brainscales-2 accelerated neuromorphic system with hybrid plasticity.
\newblock \emph{Frontiers in Neuroscience}, 16:\penalty0 795876, 2022.

\bibitem[Pei et~al.(2019)Pei, Deng, Song, Zhao, Zhang, Wu, Wang, Zou, Wu, He, et~al.]{pei2019towards}
Pei, J., Deng, L., Song, S., Zhao, M., Zhang, Y., Wu, S., Wang, G., Zou, Z., Wu, Z., He, W., et~al.
\newblock Towards artificial general intelligence with hybrid tianjic chip architecture.
\newblock \emph{Nature}, 572\penalty0 (7767):\penalty0 106--111, 2019.

\bibitem[Qiu et~al.(2024)Qiu, Zhu, Chou, Wang, Deng, and Li]{DBLP:conf/aaai/QiuZC0DL24}
Qiu, X., Zhu, R., Chou, Y., Wang, Z., Deng, L., and Li, G.
\newblock Gated attention coding for training high-performance and efficient spiking neural networks.
\newblock In \emph{Thirty-Eighth {AAAI} Conference on Artificial Intelligence, {AAAI} 2024, Thirty-Sixth Conference on Innovative Applications of Artificial Intelligence, {IAAI} 2024, Fourteenth Symposium on Educational Advances in Artificial Intelligence, {EAAI} 2014, February 20-27, 2024, Vancouver, Canada}, 2024.

\bibitem[Radosavovic et~al.(2020)Radosavovic, Kosaraju, Girshick, He, and Doll{\'{a}}r]{DBLP:conf/cvpr/RadosavovicKGHD20}
Radosavovic, I., Kosaraju, R.~P., Girshick, R.~B., He, K., and Doll{\'{a}}r, P.
\newblock Designing network design spaces.
\newblock In \emph{2020 {IEEE/CVF} Conference on Computer Vision and Pattern Recognition, {CVPR} 2020, Seattle, WA, USA, June 13-19, 2020}, 2020.

\bibitem[Rathi \& Roy(2023)Rathi and Roy]{DBLP:journals/tnn/RathiR23}
Rathi, N. and Roy, K.
\newblock {DIET-SNN:} {A} low-latency spiking neural network with direct input encoding and leakage and threshold optimization.
\newblock \emph{{IEEE} Trans. Neural Networks Learn. Syst.}, 2023.

\bibitem[Ren et~al.(2024)Ren, Zhou, Huang, Fu, Lin, Song, and Cheng]{DBLP:journals/iclr/SpikePoint}
Ren, H., Zhou, Y., Huang, Y., Fu, H., Lin, X., Song, J., and Cheng, B.
\newblock Spikepoint: An efficient point-based spiking neural network for event cameras action recognition.
\newblock 2024.

\bibitem[Roy et~al.(2019)Roy, Jaiswal, and Panda]{roy2019towards}
Roy, K., Jaiswal, A., and Panda, P.
\newblock Towards spike-based machine intelligence with neuromorphic computing.
\newblock \emph{Nature}, 575\penalty0 (7784):\penalty0 607--617, 2019.

\bibitem[Rueckauer et~al.(2017{\natexlab{a}})Rueckauer, Lungu, Hu, Pfeiffer, and Liu]{rueckauer2016theory}
Rueckauer, B., Lungu, I.-A., Hu, Y., Pfeiffer, M., and Liu, S.-C.
\newblock Conversion of continuous-valued deep networks to efficient event-driven networks for image classification.
\newblock \emph{Frontiers in Neuroscience}, 11, 2017{\natexlab{a}}.
\newblock ISSN 1662-453X.
\newblock \doi{10.3389/fnins.2017.00682}.
\newblock URL \url{https://www.frontiersin.org/articles/10.3389/fnins.2017.00682}.

\bibitem[Rueckauer et~al.(2017{\natexlab{b}})Rueckauer, Lungu, Hu, Pfeiffer, and Liu]{rueckauer2017conversion}
Rueckauer, B., Lungu, I.-A., Hu, Y., Pfeiffer, M., and Liu, S.-C.
\newblock Conversion of continuous-valued deep networks to efficient event-driven networks for image classification.
\newblock \emph{Frontiers in neuroscience}, 11:\penalty0 682, 2017{\natexlab{b}}.

\bibitem[Sengupta et~al.(2018)Sengupta, Ye, Wang, Liu, and Roy]{Sengupta2018Going}
Sengupta, A., Ye, Y., Wang, R., Liu, C., and Roy, K.
\newblock Going deeper in spiking neural networks: Vgg and residual architectures.
\newblock \emph{Frontiers in Neuroence}, 2018.

\bibitem[Shadlen \& Newsome(1994)Shadlen and Newsome]{shadlen1994noise}
Shadlen, M.~N. and Newsome, W.~T.
\newblock Noise, neural codes and cortical organization.
\newblock \emph{Current opinion in neurobiology}, 4\penalty0 (4):\penalty0 569--579, 1994.

\bibitem[Shen et~al.(2024)Shen, Zhao, Li, Li, and Zeng]{shen2024conventional}
Shen, G., Zhao, D., Li, T., Li, J., and Zeng, Y.
\newblock Are conventional snns really efficient? a perspective from network quantization.
\newblock In \emph{Proceedings of the IEEE/CVF Conference on Computer Vision and Pattern Recognition}, pp.\  27538--27547, 2024.

\bibitem[Simonyan \& Zisserman(2015)Simonyan and Zisserman]{simonyan2015deep}
Simonyan, K. and Zisserman, A.
\newblock Very deep convolutional networks for large-scale image recognition, 2015.

\bibitem[Softky \& Koch(1993)Softky and Koch]{softky1993highly}
Softky, W.~R. and Koch, C.
\newblock The highly irregular firing of cortical cells is inconsistent with temporal integration of random epsps.
\newblock \emph{Journal of neuroscience}, 13\penalty0 (1):\penalty0 334--350, 1993.

\bibitem[Song et~al.(2021)Song, Varshika, Das, and Kandasamy]{song2021design}
Song, S., Varshika, M.~L., Das, A., and Kandasamy, N.
\newblock A design flow for mapping spiking neural networks to many-core neuromorphic hardware.
\newblock In \emph{2021 IEEE/ACM International Conference On Computer Aided Design (ICCAD)}, pp.\  1--9. IEEE, 2021.

\bibitem[Stein et~al.(2005)Stein, Gossen, and Jones]{stein2005neuronal}
Stein, R.~B., Gossen, E.~R., and Jones, K.~E.
\newblock Neuronal variability: noise or part of the signal?
\newblock \emph{Nature Reviews Neuroscience}, 6\penalty0 (5):\penalty0 389--397, 2005.

\bibitem[St{\"o}ckl \& Maass(2021)St{\"o}ckl and Maass]{stockl2021optimized}
St{\"o}ckl, C. and Maass, W.
\newblock Optimized spiking neurons can classify images with high accuracy through temporal coding with two spikes.
\newblock \emph{Nature Machine Intelligence}, 3\penalty0 (3):\penalty0 230--238, 2021.

\bibitem[Varshika et~al.(2022)Varshika, Balaji, Corradi, Das, Stuijt, and Catthoor]{varshika2022design}
Varshika, M.~L., Balaji, A., Corradi, F., Das, A., Stuijt, J., and Catthoor, F.
\newblock Design of many-core big little $\mu$brains for energy-efficient embedded neuromorphic computing.
\newblock In \emph{2022 Design, Automation \& Test in Europe Conference \& Exhibition (DATE)}, pp.\  1011--1016. IEEE, 2022.

\bibitem[Wang et~al.(2023{\natexlab{a}})Wang, Zhang, Han, Wang, Zhang, and Xu]{DBLP:conf/aaai/WangZHWZX23}
Wang, Q., Zhang, T., Han, M., Wang, Y., Zhang, D., and Xu, B.
\newblock Complex dynamic neurons improved spiking transformer network for efficient automatic speech recognition.
\newblock In \emph{Thirty-Seventh {AAAI} Conference on Artificial Intelligence, {AAAI} 2023, Thirty-Fifth Conference on Innovative Applications of Artificial Intelligence, {IAAI} 2023, Thirteenth Symposium on Educational Advances in Artificial Intelligence, {EAAI} 2023, Washington, DC, USA, February 7-14, 2023}, 2023{\natexlab{a}}.

\bibitem[Wang et~al.(2022{\natexlab{a}})Wang, Zhang, Chen, and Qu]{wang2022signed}
Wang, Y., Zhang, M., Chen, Y., and Qu, H.
\newblock Signed neuron with memory: Towards simple, accurate and high-efficient ann-snn conversion.
\newblock In Raedt, L.~D. (ed.), \emph{Proceedings of the Thirty-First International Joint Conference on Artificial Intelligence, {IJCAI-22}}, pp.\  2501--2508. International Joint Conferences on Artificial Intelligence Organization, 7 2022{\natexlab{a}}.
\newblock \doi{10.24963/ijcai.2022/347}.
\newblock URL \url{https://doi.org/10.24963/ijcai.2022/347}.
\newblock Main Track.

\bibitem[Wang et~al.(2022{\natexlab{b}})Wang, Lian, Zhang, Cui, Yan, and Tang]{DBLP:journals/corr/abs-2205-07473}
Wang, Z., Lian, S., Zhang, Y., Cui, X., Yan, R., and Tang, H.
\newblock Towards lossless {ANN-SNN} conversion under ultra-low latency with dual-phase optimization.
\newblock \emph{CoRR}, 2022{\natexlab{b}}.

\bibitem[Wang et~al.(2023{\natexlab{b}})Wang, Fang, Cao, Zhang, Wang, and Xu]{wang2023masked}
Wang, Z., Fang, Y., Cao, J., Zhang, Q., Wang, Z., and Xu, R.
\newblock Masked spiking transformer.
\newblock In \emph{Proceedings of the IEEE/CVF International Conference on Computer Vision}, 2023{\natexlab{b}}.

\bibitem[Wang et~al.(2025)Wang, Fang, Cao, Ren, and Xu]{wang2024adaptive}
Wang, Z., Fang, Y., Cao, J., Ren, H., and Xu, R.
\newblock Adaptive calibration: A unified conversion framework of spiking neural network.
\newblock In \emph{Proceedings of the AAAI Conference on Artificial Intelligence}, 2025.

\bibitem[Wei et~al.(2023)Wei, Zhang, Qu, Belatreche, Zhang, and Chen]{wei2023temporal}
Wei, W., Zhang, M., Qu, H., Belatreche, A., Zhang, J., and Chen, H.
\newblock Temporal-coded spiking neural networks with dynamic firing threshold: Learning with event-driven backpropagation.
\newblock In \emph{Proceedings of the IEEE/CVF International Conference on Computer Vision}, 2023.

\bibitem[Wightman(2019)]{rw2019timm}
Wightman, R.
\newblock Pytorch image models.
\newblock \url{https://github.com/rwightman/pytorch-image-models}, 2019.

\bibitem[Wu et~al.(2023)Wu, Chua, Zhang, Li, Li, and Tan]{DBLP:journals/tnn/WuCZLLT23}
Wu, J., Chua, Y., Zhang, M., Li, G., Li, H., and Tan, K.~C.
\newblock A tandem learning rule for effective training and rapid inference of deep spiking neural networks.
\newblock \emph{{IEEE} Trans. Neural Networks Learn. Syst.}, 2023.

\bibitem[Wu et~al.(2024)Wu, Bojkovic, Gu, Suo, and Zou]{wu2024ftbc}
Wu, X., Bojkovic, V., Gu, B., Suo, K., and Zou, K.
\newblock Ftbc: Forward temporal bias correction for optimizing ann-snn conversion.
\newblock \emph{ECCV}, 2024.

\bibitem[Wu et~al.(2018)Wu, Deng, Li, Zhu, and Shi]{wu2018spatio}
Wu, Y., Deng, L., Li, G., Zhu, J., and Shi, L.
\newblock Spatio-temporal backpropagation for training high-performance spiking neural networks.
\newblock \emph{Frontiers in neuroscience}, 12:\penalty0 331, 2018.

\bibitem[Yang et~al.(2019)Yang, Wu, Wang, Yang, Li, Deng, Zhu, and Shi]{yang2019dashnet}
Yang, Z., Wu, Y., Wang, G., Yang, Y., Li, G., Deng, L., Zhu, J., and Shi, L.
\newblock Dashnet: a hybrid artificial and spiking neural network for high-speed object tracking.
\newblock \emph{arXiv preprint arXiv:1909.12942}, 2019.

\bibitem[Yao et~al.(2021)Yao, Gao, Zhao, Wang, Lin, Yang, and Li]{yaoTemporalwiseAttentionSpiking2021}
Yao, M., Gao, H., Zhao, G., Wang, D., Lin, Y., Yang, Z., and Li, G.
\newblock Temporal-wise attention spiking neural networks for event streams classification.
\newblock In \emph{Proceedings of the {{IEEE}}/{{CVF International Conference}} on {{Computer Vision}}}, pp.\  10221--10230, 2021.

\bibitem[Zenke \& Ganguli(2018)Zenke and Ganguli]{zenke2018superspike}
Zenke, F. and Ganguli, S.
\newblock Superspike: Supervised learning in multilayer spiking neural networks.
\newblock MIT Press One Rogers Street, Cambridge, MA 02142-1209, USA journals-info~…, 2018.

\bibitem[Zenke \& Vogels(2021)Zenke and Vogels]{zenke2021remarkable}
Zenke, F. and Vogels, T.~P.
\newblock The remarkable robustness of surrogate gradient learning for instilling complex function in spiking neural networks.
\newblock \emph{Neural computation}, 33\penalty0 (4):\penalty0 899--925, 2021.

\bibitem[Zhang \& Zhang(2024)Zhang and Zhang]{zhang2023memory}
Zhang, H. and Zhang, Y.
\newblock Memory-efficient reversible spiking neural networks.
\newblock In \emph{Proceedings of the AAAI conference on artificial intelligence}, 2024.

\bibitem[Zhou et~al.(2022)Zhou, Zhu, He, Wang, Yan, Tian, and Yuan]{zhouSpikformerWhenSpiking2022}
Zhou, Z., Zhu, Y., He, C., Wang, Y., Yan, S., Tian, Y., and Yuan, L.
\newblock Spikformer: {{When Spiking Neural Network Meets Transformer}}.
\newblock \emph{arXiv preprint arXiv:2209.15425}, 2022.

\bibitem[Zhu et~al.(2022)Zhu, Wang, Chang, Li, Huang, and Tian]{zhu2022event}
Zhu, L., Wang, X., Chang, Y., Li, J., Huang, T., and Tian, Y.
\newblock Event-based video reconstruction via potential-assisted spiking neural network.
\newblock In \emph{Proceedings of the IEEE/CVF Conference on Computer Vision and Pattern Recognition}, pp.\  3594--3604, 2022.

\bibitem[Zhu et~al.(2023)Zhu, Zhao, and Eshraghian]{DBLP:journals/corr/abs-2302-13939}
Zhu, R., Zhao, Q., and Eshraghian, J.~K.
\newblock Spikegpt: Generative pre-trained language model with spiking neural networks.
\newblock \emph{CoRR}, 2023.

\bibitem[Zhu et~al.(2024)Zhu, Fang, Xie, Huang, and Yu]{zhu2024exploring}
Zhu, Y., Fang, W., Xie, X., Huang, T., and Yu, Z.
\newblock Exploring loss functions for time-based training strategy in spiking neural networks.
\newblock 2024.

\end{thebibliography}
\bibliographystyle{icml2025}

\newpage
\appendix
\setcounter{theorem}{0} 

\onecolumn

\section{Conversion steps}\label{app conversion steps}


\paragraph{Copying ANN architecture and weights.} ANN-SNN conversion process starts with a pre-trained ANN model, whose weights (and biases) will be copied to an SNN model following the same architecture. In this process, one considers ANN models whose non-activation layers become linear during the inference. In particular, these include fully connected, convolutional, batch normalization and average pooling layers.  

\paragraph{Approximating ANN activation functions.} The second step of the process considers the activation layers and their activation functions in ANN. Here, the idea is to initialize the spiking neurons in the corresponding SNN layer in such a way that their average spiking rate approximates the values of the corresponding activation functions. For the \relu (or \relu-like such as quantized or thresholded \relu) activations, this process is rather well understood. The spiking neuron threshold is usually set to correspond to the maximum activation ANN channel or layerwise, or to be some percentile of it. If we denote by $f$ the ANN actiavtion, then ideally, after setting the thresholds, one would like to have
\begin{equation}\label{eq main approx}
f(\bv[T])\approx \frac{\theta}{T}\cdot\sum_{t=1}^T \bs[t].
\end{equation} 
 
 If we recall the equations for the IF neuron (equations \eqref{eq: dynamics} in the article)
 \begin{align}\label{eq dynamics help}
\bv^\lup[t] &= \bv^\lup[t-1] +\bW^\lup \theta^{(l-1)}\cdot \bs^{(l-1)}[t]-\theta^\lup \cdot\bs[t-1],\\
\bs^\lup[t] &= H(\bv^\lup[t]- \theta^\lup),
\end{align}
we see that the value with which we are comparing the membrane potential (threshold) is the same as the value with which we are scaling the output spikes. In particular, as soon as our membrane potential has reached $\theta$, it will produce the value $\theta$. This can be loosely described as, whatever the input is, the output will be approximately that value (or zero, if the input is negative), which is exactly what \relu does. 

\paragraph{Absorbing thresholds.} Finally, we notice that, once we produce a spike $\bs^\lup[t]$, the value $\theta^\lup\cdot \bs^\lup[t]$ will be sent to the next layer, and will further be weighted with weights $W^{(l+1)}$ and the bias $b^{(l+1)}$ will be applied. As we want SNNs to operate only using ones and zeros (to avoid multiplication due to energy efficiency), the values $\theta^\lup$ will be absorbed into $W^{(l+1)}$, i.e. $W^{(l+1)}\leftarrow \theta^\lup W^{(l+1)}$.

\section{Proof of the theoretical results}
\label{appendix-sec:conversion_error_analysis}
We prove the main theorem from the article, which we restate here. 
\mainthm*

\begin{proof}
    We start by rewriting equation \eqref{eq probabilistic spiking} as 
    \begin{align}
    \bs^\lup[t] &= B\left(\frac{1}{\theta^\lup\cdot (T-t+1)}\cdot\bv^\lup[t-1]\right)\\
    &= B\left(\frac{1}{\theta^\lup\cdot (T-t+1)}\cdot\left(T\cdot X^\lup-\theta^\lup\cdot\sum_{i=1}^{t-1}s^\lup[i]\right)\right)\\
    &= B\left(\frac{1}{T-t+1}\cdot\left(T\cdot \frac{X^\lup}{\theta^\lup}-\sum_{i=1}^{t-1}s^\lup[i]\right)\right),\label{eq last expr}
    \end{align}
    obtained by unrolling through time the expression for the membrane potential. 

    We next consider various settings of the theorem.
For the first three statements, we can assume that the vector $X^\lup$ is one dimensional, i.e. $X^\lup\in\mathbb{R}$. We start by first considering the situation where $X^\lup>\theta^\lup$, so that the output of the ANN is $\relu_{\theta^\lup}(X^\lup)=\theta^\lup$. In that case, we notice that the spiking neuron will fire at every time step $t=1,\dots, T$, because the bias for the Bernoulli random variable will always be bigger than or equal to 1, as follows from the previous equations. Similarly, if $X^\lup\leq0$, the bias will always be non-positive, and both the output of ANN and of SNN will be 0. Consequently, the first three statements follow directly for both of these cases.

Suppose now that $0\leq X^\lup<\theta^\lup$ (or, equivalently, $0\leq\frac{X^\lup}{\theta^\lup}<1$) and let $a$ be minimal non-negative integer such that $a\cdot\theta^\lup\leq T\cdot X^\lup<(a+1)\cdot \theta$, i.e. $a:=\lfloor \frac{T\cdot X^\lup}{\theta^\lup}\rfloor$. In particular, $0\leq a<T$. We proceed to prove statement $(b)$ above. We note that in general, $\sum_{i=1}^{t-1}s^\lup[i]\leq a+1$, because after at most $a+1$ spikes, the residue membrane potential would become negative. Also, due to the condition of the statement that the residue membrane potential is still non-negative, we may assume that $\sum_{i=1}^{t-1}s^\lup[i]\leq a$. Then, it becomes clear that the bias in the last equation \eqref{eq last expr} is a non-negative number smaller than 1, and it follows that 
\begin{equation}
    \mathbb{E}\left[s_i^{(l)}[t]\right] = \frac{1}{T-t+1}\cdot\left(T\cdot \frac{X^\lup}{\theta^\lup}-\sum_{i=1}^{t-1}s^\lup[i]\right),
\end{equation}
so that we finally have 
\begin{align}
    \frac{\theta^\lup\cdot (T-t+1)}{T}&\cdot\mathbb{E}\left[s_i^{(l)}[t]\right]+\frac{\theta^\lup}{T}\cdot\sum_{i=1}^{t-1} s_i^{(l)}[i]\\
    &=\frac{\theta^\lup\cdot (T-t+1)}{T}\cdot\frac{1}{T-t+1}\cdot\left(T\cdot \frac{X^\lup}{\theta^\lup}-\sum_{i=1}^{t-1}s^\lup[i]\right)+\frac{\theta^\lup}{T}\cdot\sum_{i=1}^{t-1} s_i^{(l)}[i]\\
    &= X^\lup = \relu_{\theta^\lup}(X^\lup).
\end{align}
Statement $(a)$ follows from $(b)$, while for the statement $(c)$ for the first case, we notice that after every spike, the residue membrane potential is a multiple of $\theta^\lup$, hence we will have exactly $a$ spikes (notation above), while the second case follows from statement $(b)$. Finally, statement $(d)$ is a direct consequence of the above discussion. 
\end{proof}



\section{Experiments Details}

\subsection{Datasets}

\noindent \textbf{CIFAR-10}: The CIFAR-10 dataset~\cite{krizhevsky2010cifar} contains 60,000 color images of 32x32 pixels each, divided into 10 distinct classes (e.g., airplanes, cars, birds), with each class containing 6,000 images. The dataset is split into 50,000 training images and 10,000 test images.

\noindent \textbf{CIFAR-100}: The CIFAR-100 dataset~\cite{krizhevsky2010cifar} consists of 60,000 color images of 32x32 pixels, distributed across 100 classes, with each class having 600 images. Similar to CIFAR-10, it is divided into 50,000 training images and 10,000 test images.


\noindent \textbf{CIFAR10-DVS}: The CIFAR10-DVS~\cite{li2017cifar10} dataset is a neuromorphic vision benchmark derived from the CIFAR-10 dataset, converting 10,000 static images (1,000 per class across 10 categories like airplanes and cars) into event streams using a Dynamic Vision Sensor (DVS). Generated via Repeated Closed-Loop Smooth (RCLS) movement to simulate realistic motion, the dataset captures spatio-temporal intensity changes as asynchronous ON/OFF events (128×128 resolution) with inherent noise (e.g., 60Hz LCD artifacts, later corrected). Designed for moderate complexity between MNIST-DVS and N-Caltech101, it challenges event-driven algorithms, achieving initial low accuracies (~22–29\%) with methods like spiking neural networks and SVM, highlighting its utility for advancing neuromorphic research. Publicly available on Figshare\footnote{\url{https://figshare.com/articles/dataset/CIFAR10-DVS_New/4724671}}, CIFAR10-DVS bridges traditional and event-based vision, fostering innovation in brain-inspired computing and real-world dynamic scene analysis.

\noindent \textbf{ImageNet}: The ImageNet dataset~\cite{deng2009imagenet} comprises 1,281,167 images spanning 1,000 classes in the training set, with a validation set and a test set containing 50,000 and 100,000 images, respectively. Unlike the CIFAR datasets, ImageNet images vary in size and resolution. The validation set is frequently used as the test set in various applications.

\subsection{Configuration and Setups}
\label{appendix:config}

\subsubsection{Ours + QCFS}

\noindent \textbf{CIFAR}: We followed the original paper's training configurations to train ResNet-20 and VGG-16 on CIFAR-100. The Stochastic Gradient Descent (SGD) optimizer with a momentum of 0.9 was used. The initial learning rate was set to 0.02, with a weight decay of $5 \times 10^{-4}$. A cosine decay scheduler adjusted the learning rate over 300 training epochs. The quantization steps $L$ were set to 8 for ResNet-20 and 4 for VGG-16. All models were trained for 300 epochs.

\noindent \textbf{ImageNet}: We utilized checkpoints for ResNet-34 and VGG-16 from the original paper's GitHub repository. For ImageNet, $L$ was set to 8 and 16 for ResNet-34 and VGG-16, respectively. 

\subsubsection{Ours + RTS}
\noindent \textbf{CIFAR}: We trained models using the recommended settings from the original paper.

\noindent \textbf{ImageNet}: We used pre-trained checkpoints for ResNet-34 and VGG-16 from the original paper's GitHub repository. Subsequently, all ReLU layers were replaced with spiking neuron layers.

For all datasets, we initialize TPP membrane potential to zero, while in the baselines we do as they propose.

\subsubsection{Ours + SNNC w/o Calibration}

\noindent \textbf{CIFAR}: We adhered to the original paper's configurations to train ResNet-20 and VGG-16 on CIFAR-100. The SGD optimizer with a momentum of 0.9 was used. The initial learning rate was set to 0.01, with a weight decay of $5 \times 10^{-4}$ for models with batch normalization. A cosine decay scheduler adjusted the learning rate over 300 training epochs. All models were trained for 300 epochs with a batch size of 128.

\noindent \textbf{ImageNet}: We used pre-trained checkpoints for ResNet-34 and VGG-16 from the original paper's GitHub repository. Subsequently, all ReLU layers were replaced with our proposed spiking neuron layers.

\clearpage

\section{Algorithms}
\label{appendix:algo}
The baseline SNN neuron forward function (Algorithm \ref{algo:snn-forward}) initializes the membrane potential to zero and iteratively updates it by adding the layer output at each timestep. Spikes are generated when the membrane potential exceeds a defined threshold, $\theta$, and the potential is reset accordingly. This function captures the core dynamics of spiking neurons. The Shuffle Mode (Algorithm \ref{algo:snn-shuffle-forward}) is an extension of the baseline forward function. After generating the spikes across the simulation length, this mode shuffles the spike train. 

The TPP Mode (Algorithm \ref{algo:snn-tpp-forward}) introduces a probabilistic component to the spike generation process. Instead of a deterministic threshold-based spike generation, it uses a Bernoulli process where the probability of spiking is determined by the current membrane potential relative to the threshold adjusted for the remaining timesteps. 


\begin{algorithm}[H]
\caption{SNN Neuron Forward Function and Additional Modes}
\label{algo:snn-forward}
\begin{algorithmic}[1]
\REQUIRE SNN Layer $\mathcal{\ell}$; Input tensor $\mathbf{x}$; Threshold $\theta$; Simulation length $T$.

\FUNCTION{\textsc{BaselineSNN}($\mathcal{\ell}, \mathbf{x}, \theta, T$)}
    \STATE $\mathbf{v} \gets 0$ \COMMENT{Initialize membrane potential}
    \FOR{$t = 1$ \textbf{to} $T$}
        \STATE $\mathbf{v} \gets \mathbf{v} + \mathcal{\ell}(\mathbf{x}(t))$
        \STATE $\mathbf{s} \gets (\mathbf{v} \geq \theta) \times \theta$
        \STATE $\mathbf{v} \gets \mathbf{v} - \mathbf{s}$
        \STATE Store $\mathbf{s}(t)$
    \ENDFOR
    \STATE \textbf{return} $\mathbf{s}$
\ENDFUNCTION

\end{algorithmic}
\end{algorithm}

\begin{algorithm}[H]
\caption{SNN Neuron Forward Function of Shuffle Mode}
\label{algo:snn-shuffle-forward}
\begin{algorithmic}[1]
\REQUIRE SNN Layer $\mathcal{\ell}$; Input tensor $\mathbf{x}$; Threshold $\theta$; Simulation length $T$.

\FUNCTION{\textsc{ShuffleMode}($\mathcal{\ell}, \mathbf{x}, \theta, T$)}
    \STATE $\mathbf{v} \gets 0$ \COMMENT{Initialize membrane potential}
    \FOR{$t = 1$ \textbf{to} $T$}
        \STATE $\mathbf{v} \gets \mathbf{v} + \mathcal{\ell}(\mathbf{x}(t))$
        \STATE $\mathbf{s} \gets (\mathbf{v} \geq \theta) \times \theta$
        \STATE $\mathbf{v} \gets \mathbf{v} - \mathbf{s}$
        \STATE Store $\mathbf{s}(t)$
    \ENDFOR
    \STATE Shuffle the stored spikes $\mathbf{s}(1), \mathbf{s}(2), \ldots, \mathbf{s}(T)$
    \STATE \textbf{return} shuffled $\mathbf{s}$
\ENDFUNCTION

\end{algorithmic}
\end{algorithm}

\begin{algorithm}[H]
\caption{SNN Neuron Forward Function of TPP Mode}
\label{algo:snn-tpp-forward}
\begin{algorithmic}[1]
\REQUIRE SNN Layer $\mathcal{\ell}$; Input tensor $\mathbf{x}$; Threshold $\theta$; Simulation length $T$.

\FUNCTION{\textsc{TPPMode}($\mathcal{\ell}, \mathbf{x}, \theta, T$)}
    \STATE $\mathbf{v} \gets \sum_{t=1}^{T} \mathbf{x}(t)$ \COMMENT{Initialize membrane potential with the sum of inputs}
    \FOR{$t = 1$ \textbf{to} $T$}
        \STATE $\mathbf{p} \gets \text{Clamp}(\mathbf{v} / (\theta \times (T - t + 1)), 0, 1)$
        \STATE $\mathbf{s} \gets \text{Bernoulli}(\mathbf{p}) \times \theta$
        \STATE $\mathbf{v} \gets \mathbf{v} - \mathbf{s}$
        \STATE Store $\mathbf{s}(t)$
    \ENDFOR
    \STATE \textbf{return} $\mathbf{s}$
\ENDFUNCTION
\end{algorithmic}
\end{algorithm}

\section{Additional Experiments}

\subsection{SNNC} 

We show extra experiment results about the comparison among permutation method and two-phase probabilistic method. We validated ResNet-20 and VGG-16 on the CIFAR-10/100 dataset , and ResNet-34, VGG-16 and RegNetX-4GF on ImageNet with batch and channel-wise normalization enabled. Using a batch size of 128, the experiment was run five times with different random seeds to ensure reliable and reproducible results.

\begin{table}[H]
\caption{Comparison between our proposed methods and ANN-SNN conversion SNNC method on \textbf{CIFAR-10}. The average accuracy and standard deviation of the TPP method are reported over 5 experiments.}
\label{tab:ann-snn-snnc-cifar10}
\centering
\scalebox{0.8}
{
\begin{threeparttable}
\begin{tabular}{@{}cccccccccc@{}}
\toprule
Architecture & Method & ANN & T=1 & T=2 & T=4 & T=8 & T=16 & T=32 & T=64 \\ 
\toprule
\multirow{3}{*}{ResNet-20}
& SNNC-AP~\cite{li2021free} & 96.95 & 51.20 & 66.07 & 83.60 & 92.79 & 95.62 & 96.58 & 96.85 \\
& \textbf{Ours (Permute)} & 96.95 & 34.05 & 61.46 & 90.54 & 95.05 & 96.12 & 96.62 & 96.77 \\
& \textbf{Ours (TPP)} & 96.95 & 10.05 (0.02) & 17.30 (0.52)	& 79.19 (0.67) & 93.72 (0.05) & 95.87 (0.09) & 96.67 (0.04) & 96.80 (0.01) \\
\midrule
\multirow{4}{*}{VGG-16}
& SNNC-AP~\cite{li2021free} & 95.69 & 60.72 & 75.82 & 82.18 & 91.93 & 93.27 & 94.97 & 95.40 \\
& \textbf{Ours (Permute)} & 95.69 & 38.01 & 64.40 & 84.65 & 92.24 & 92.80 & 93.33 & 94.10 \\
& \textbf{Ours (TPP)} & 95.69 & 11.46 (0.35) & 32.24 (1.40) & 86.85 (0.42) & 94.34 (0.12) & 94.86 (0.06) & 95.48 (0.03) & 95.60 (0.04) \\
\bottomrule
\end{tabular}
\end{threeparttable}
}
\end{table}


\begin{table}[H]
\caption{Comparison between our proposed methods and ANN-SNN conversion SNNC method on \textbf{CIFAR-100}. The average accuracy and standard deviation of the TPP method are reported over 5 experiments.}
\label{tab:ann-snn-snnc-cifar100}
\centering
\scalebox{0.8}
{
\begin{threeparttable}
\begin{tabular}{@{}cccccccccc@{}}
\toprule
Architecture & Method & ANN & T=1 & T=2 & T=4 & T=8 & T=16 & T=32 & T=64 \\ 
\toprule
\multirow{3}{*}{ResNet-20}
& SNNC-AP~\cite{li2021free} & 81.89 & 17.91 & 34.08 &	54.78 & 72.28 & 78.57 & 81.20 & 81.95 \\
& \textbf{Ours (Permute)} & 81.89 &	5.64 & 19.54 & 52.46 & 75.21 & 79.76 & 81.12 & 81.52 \\
& \textbf{Ours (TPP)} & 81.89 &	1.94 (0.11) & 5.15 (0.44) & 39.67 (0.99) & 71.05 (0.68) & 78.97 (0.24) & 81.06 (0.05) & 81.61 (0.08) \\
\midrule
\multirow{4}{*}{VGG-16}
& SNNC-AP~\cite{li2021free} & 77.87 & 28.64 & 34.87 & 50.95 & 64.30 & 71.93 & 75.39 & 77.05 \\
& \textbf{Ours (Permute)} & 77.87 & 12.50 & 34.98 & 60.81 & 69.42 & 72.78 & 73.50 & 75.14 \\
& \textbf{Ours (TPP)} & 77.87 & 2.05 (0.27) & 15.90 (0.71) & 59.23 (0.65) & 73.16 (0.17) & 76.05 (0.26) & 77.16 (0.09) & 77.56 (0.13) \\
\bottomrule
\end{tabular}
\end{threeparttable}
}
\end{table}


\begin{table}[H]
\caption{Comparison between our proposed methods and ANN-SNN conversion SNNC method on ImageNet. The average accuracy and standard deviation of the TPP method are reported over 5 experiments.}
\label{tab:ann-snn-snnc-imagenet}
\centering
\scalebox{0.8}
{
\begin{threeparttable}
\begin{tabular}{@{}cccccccccc@{}}
\toprule
Architecture & Method & ANN & T=4 & T=8 & T=16 & T=32 & T=64 & T=128 \\ 
\toprule
\multirow{3}{*}{ResNet-34}
& SNNC-AP~\cite{li2021free} & 75.65 & -- & -- &	-- & 64.54 & 71.12 & 73.45 \\
& \textbf{Ours (Permute)} & 75.65 &	10.51 & 57.57 & 70.94 & 74.00 & 75.06 & 75.47 \\
& \textbf{Ours (TPP)} & 75.65 & 2.69 (0.03) & 49.24 (0.23) & 69.97 (0.10) & 74.07 (0.06) & 75.23 (0.03) & 75.51 (0.05) \\
\midrule
\multirow{4}{*}{VGG-16}
& SNNC-AP~\cite{li2021free} & 75.37 & -- & -- & -- & 63.64 & 70.69 & 73.32 \\
& \textbf{Ours (Permute)} & 75.37 & 38.61 & 67.29 & 73.35 & 74.34 & 74.82 & 75.11 \\
& \textbf{Ours (TPP)} & 75.37 & 54.14 (0.59) & 69.75 (0.27) & 73.44 (0.02) & 74.72 (0.06) & 75.14 (0.02) & 75.25 (0.03) \\
\midrule
\multirow{4}{*}{RegNetX-4GF}
& SNNC-AP~\cite{li2021free} & 80.02 & -- & -- & -- & 55.70 & 70.96 & 75.78 \\
& \textbf{Ours (Permute)} & 78.45 & -- & -- & 43.45 & 68.12 & 75.63 & 77.63 \\
& \textbf{Ours (TPP)} & 78.45 & -- & -- & 22.71 (2.98) & 66.51 (0.44) & 75.54 (0.07) & 77.83 (0.04) \\
\bottomrule
\end{tabular}
\end{threeparttable}
}
\end{table}

\subsection{RTS} 

\begin{table}[H]
\caption{Comparison between our proposed methods and ANN-SNN conversion RTS method on CIFAR-10/100 and ImageNet. The average accuracy and standard deviation of the TPP method are reported over 5 experiments.}
\label{tab:ann-snn-rts-cifar10}
\centering
\scalebox{0.75}
{
\begin{threeparttable}
\begin{tabular}{@{}ccccccccccc@{}}
\toprule
\textbf{Dataset} & \textbf{Architecture} & \textbf{Method} & \textbf{ANN} & \textbf{T=4} & \textbf{T=8} & \textbf{T=16} & \textbf{T=32} & \textbf{T=64} & \textbf{T=128} \\ 
\toprule
\multirow{6}{*}{CIFAR-10}
& \multirow{3}{*}{VGG-16} 
& RTS\tnote{*}~\cite{DBLP:conf/iclr/DengG21} & 94.99 & 88.64 & 91.67 & 93.64 & 94.50 & 94.76 & 94.91 \\
& & \textbf{Ours (Permute)} & 94.99 & 91.22 & 93.70 & 94.50 & 94.86 & 94.88 & 94.97 \\
& & \textbf{Ours (TPP)} & 94.99 & 91.49 (0.21) & 94.11 (0.09) & 94.72 (0.08) & 94.84 (0.06) & 94.91 (0.02) & 94.98 (0.02) \\
\cline{2-10}
& \multirow{3}{*}{ResNet-20}
& RTS\tnote{*}~\cite{DBLP:conf/iclr/DengG21} & 91.07 & 27.08 & 40.88 & 65.13 & 84.75 & 90.12 & 90.76 \\
& & \textbf{Ours (Permute)} & 91.07 & 68.18 & 86.57 & 90.20 & 90.81 & 91.04 & 90.99 \\
& & \textbf{Ours (TPP)} & 91.07 & 72.87 (0.22) & 88.27 (0.14) & 90.44 (0.08) & 90.86 (0.14) & 90.94 (0.04) & 91.01 (0.03) \\
\cline{1-10}
\multirow{3}{*}{CIFAR-100} 
& \multirow{3}{*}{VGG-16}
& RTS\tnote{$\circ$}~\cite{DBLP:conf/iclr/DengG21} & 76.13 & 23.76 & 43.81 & 56.23 & 67.61 & 73.45 & 75.23 \\
& & \textbf{Ours (Permute)} & 76.13 & 35.31 & 62.84 & 71.20 & 74.34 & 75.53 & 75.92 \\
& & \textbf{Ours (TPP) + RTS} & 76.13 & 37.88 (0.35) & 65.81 (0.27) & 73.05 (0.12) & 75.17 (0.17) & 75.64 (0.12) & 75.90 (0.08) \\
\cline{1-10}
\multirow{3}{*}{ImageNet} 
& \multirow{3}{*}{VGG-16}
& RTS~\cite{DBLP:conf/iclr/DengG21} & 72.16 & -- & -- & 55.80 & 67.73 & 70.97 & 71.89 \\
& & \textbf{Ours (Permute)} & 72.16 & 33.77 & 58.31 & 67.80 & 70.89 & 71.65 & 71.95 \\
& & \textbf{Ours (TPP)} & 72.16 & 30.50 (1.19) & 56.69(0.67) & 67.34 (0.25) & 70.63 (0.11) & 71.75 (0.05) & 72.05 (0.03) \\
\bottomrule
\end{tabular}
\end{threeparttable}
}
\end{table}

\subsection{QCFS} 

\begin{table}[H]
\caption{Comparison between our proposed methods and ANN-SNN conversion QCFS method on CIFAR-10/100 and ImageNet. The average accuracy and standard deviation of the TPP method are reported over 5 experiments.}
\label{tab:ann-snn-qcfs-cifar10}
\centering
\scalebox{0.8}
{
\begin{threeparttable}
\begin{tabular}{@{}ccccccccccc@{}}
\toprule
\textbf{Dataset} & \textbf{Architecture} & \textbf{Method} & \textbf{ANN} & \textbf{T=4} & \textbf{T=8} & \textbf{T=16} & \textbf{T=32} & \textbf{T=64} \\ 
\toprule
\multirow{6}{*}{CIFAR-10}
& \multirow{3}{*}{VGG-16} 
& QCFS\tnote{*}~\cite{bu2022optimal} & 95.76 & 94.33& 95.21 & 95.65 &95.87  &95.99  \\
& & \textbf{Ours (Permute)}& 95.76 & 95.15& 95.58 & 95.83 & 95.95 & 95.97 &\\
& & \textbf{Ours (TPP)} & 95.76 & 95.28(0.09) & 95.84(0.1) & 95.95(0.05) & 95.98(0.06) & 95.97 (0.03) \\
\cline{2-10}
& \multirow{3}{*}{ResNet-20}
& QCFS~\cite{bu2022optimal}& 92.43 &79.45 & 88.56 & 91.94 & 92.79 & 92.82 \\
& & \textbf{Ours (Permute)}&92.43  &84.85 &91.24  & 92.67 & 92.82 & 92.85 \\
& & \textbf{Ours (TPP)}& 92.43 & 86.24(0.18)&92.08(0.11)  &92.70(0.1)  &92.78(0.04  &92.68(0.06) \\
\cline{1-10}
\multirow{6}{*}{CIFAR-100} 
& \multirow{3}{*}{VGG-16}
& QCFS\tnote{$\circ$}~\cite{bu2022optimal} &76.3  &69.29 & 73.89 & 75.98 & 76.52 & 76.54 \\
& & \textbf{Ours (Permute)} & 76.3 &74.28 & 75.97 &76.54  &76.60  &76.64 \\
& & \textbf{Ours (TPP) } & 76.3 &74.0(0.15) & 76.06(0.08) &76.37(0.1)  &76.55(0.09)  &76.51(0.07) \\
\cline{2-10}
& \multirow{3}{*}{ResNet-20}
& QCFS~\cite{bu2022optimal}&67.0  &27.44 & 49.35 & 63.12 & 66.84 & 67.77 \\
& & \textbf{Ours (Permute)}&67.0  & 45.33&62.81  &66.93  &67.85  & 67.96 \\
& & \textbf{Ours (TPP)}& 67.0 & 47.0(0.2)& 64.66(0.25) & 67.28(0.12) & 67.61(0.1) & 67.77(0.06) \\
\cline{1-10}
\multirow{3}{*}{ImageNet} 
& \multirow{3}{*}{VGG-16}
& QCFS~\cite{bu2022optimal} & 74.29 & -- & -- & 50.97 & 68.47 & 72.85 \\
& & \textbf{Ours (Permute)} & 73.89 & 55.54 & 71.12 & 73.65 & 74.28 & 74.28 \\
& & \textbf{Ours (TPP)} & 74.22 & 68.39 (0.08) & 72.99 (0.05) & 73.98 (0.07) & 74.23 (0.03) & 74.29 (0.00) \\
\bottomrule
\end{tabular}
\end{threeparttable}
}
\end{table}

\clearpage

\subsection{Spiking activity}
\label{appendix:firing-counts}

The percentage difference between the baseline and our method in TPP mode is calculated as follows: $\text{Percentage Difference} = \frac{\text{Ours} - \text{Baseline}}{\text{Baseline}} \times 100 $.

\begin{figure*}[!ht]
\centering
\subfigure[RTS]{
    \includegraphics[width=0.32\linewidth]{figures/vgg16-cifar100-rts.pdf}}
\subfigure[QCFS]{
    \includegraphics[width=0.32\linewidth]{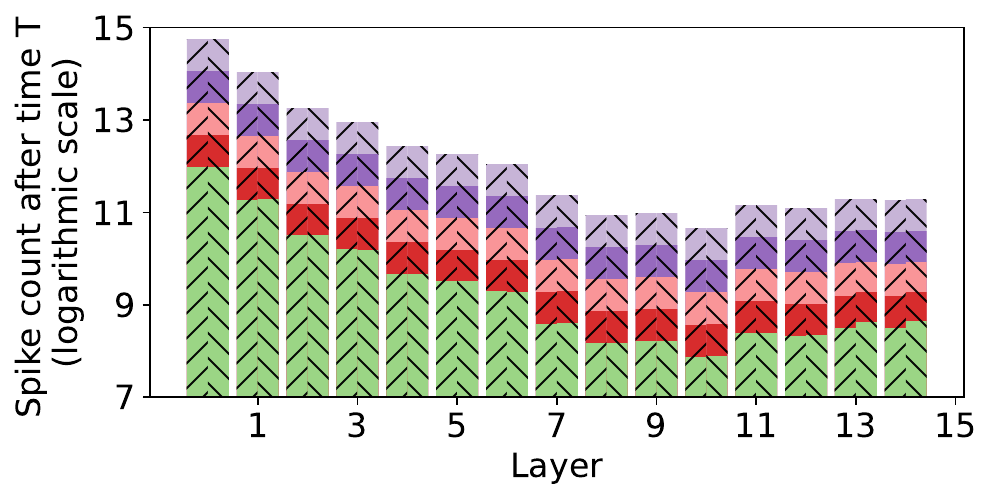}}
\subfigure[SNNC]{
    \includegraphics[width=0.32\linewidth]{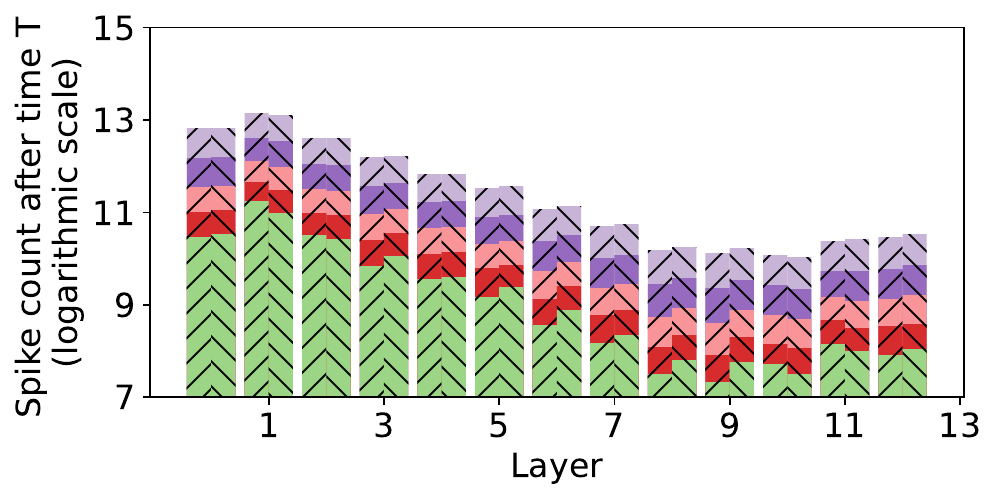}}    
\caption{Spike counts of VGG-16 on CIFAR-100 after different timesteps (T). Note: The bar height from bottom indicates the spike counts after each timestep T, and the color of longer Ts is overlaid by shorter Ts.}
 \label{fig:spike-counts-all}
\end{figure*}

\begin{table}[!htbp]
\caption{Comparison of firing counts percentage difference between the baseline and our proposed TPP method for VGG-16 on CIFAR-100 using QCFS.}
\label{tab:spike-counts-cifar100-vgg16-cifar100}
\centering
\scalebox{0.85}
{
\begin{threeparttable}
\begin{tabular}{@{}ccccccc@{}}
\toprule
Layer & T=4 & T=8 & T=16 & T=32 & T=64 & T=128 \\ 
\toprule
1 & 1.073 & 0.528 & 0.261 & 0.136 & 0.065 & 0.033 \\
\midrule
2 & 2.629 & 1.022 & 0.438 & 0.206 & 0.102 & 0.050 \\
\midrule
3 & 0.049 & 0.230 & 0.185 & 0.109 & 0.056 & 0.028 \\
\midrule
4 & -0.867 & -0.664 & -0.419 & -0.228 & -0.118 & -0.060 \\
\midrule
5 & 0.073 & 0.515 & 0.350 & 0.182 & 0.090 & 0.044 \\
\midrule
6 & 0.701 & 0.010 & -0.098 & -0.074 & -0.041 & -0.021 \\
\midrule
7 & -1.071 & -0.865 & -0.470 & -0.246 & -0.122 & -0.063 \\
\midrule
8 & 1.009 & 1.193 & 0.731 & 0.385 & 0.196 & 0.096 \\
\midrule
9 & 0.504 & 0.417 & 0.205 & 0.108 & 0.051 & 0.024 \\
\midrule
10 & -0.112 & 0.842 & 0.647 & 0.375 & 0.198 & 0.100 \\
\midrule
11 & 2.071 & 2.438 & 1.614 & 0.898 & 0.465 & 0.235 \\
\midrule
12 & 0.797 & 0.943 & 0.756 & 0.461 & 0.247 & 0.127 \\
\midrule
13 & 4.503 & 2.156 & 1.209 & 0.655 & 0.343 & 0.171 \\
\midrule
14 & 25.898 & 13.883 & 7.770 & 3.852 & 1.887 & 0.936 \\
\midrule
15 & 33.585 & 16.864 & 8.945 & 4.474 & 2.227 & 1.108 \\
\bottomrule
\end{tabular}
\end{threeparttable}
}
\end{table}

\begin{table}[!htbp]
\caption{Comparison of firing counts percentage difference between the baseline and our proposed TPP method for ResNet-34 on ImageNet using QCFS.}
\label{tab:spike-counts-cifar100-resnet34-imagenet}
\centering
\scalebox{0.85}
{
\begin{threeparttable}
\begin{tabular}{@{}ccccccc@{}}
\toprule
Layer & T=4 & T=8 & T=16 & T=32 & T=64 & T=128 \\ 
\toprule
1 & 0.587 & 0.306 & 0.149 & 0.079 & 0.036 & 0.018 \\
\midrule
2 & -0.921 & -0.435 & -0.212 & -0.108 & -0.053 & -0.025 \\
\midrule
3 & 0.353 & 0.189 & 0.082 & 0.036 & 0.019 & 0.010 \\
\midrule
4 & -2.786 & -1.583 & -0.920 & -0.506 & -0.270 & -0.141 \\
\midrule
5 & 0.469 & 0.277 & -0.107 & -0.020 & -0.019 & -0.011 \\
\midrule
6 & -3.955 & -1.865 & -0.705 & -0.344 & -0.166 & -0.086 \\
\midrule
7 & -0.381 & 0.321 & -0.090 & -0.031 & -0.020 & -0.013 \\
\midrule
8 & 6.615 & 3.261 & 1.494 & 0.628 & 0.290 & 0.131 \\
\midrule
9 & -5.116 & -3.006 & -1.555 & -0.794 & -0.391 & -0.195 \\
\midrule
10 & -2.938 & 3.431 & 3.096 & 1.794 & 0.975 & 0.498 \\
\midrule
11 & 1.184 & 0.466 & 0.359 & 0.102 & 0.053 & 0.022 \\
\midrule
12 & -17.739 & -7.302 & -1.788 & -0.609 & -0.270 & -0.132 \\
\midrule
13 & 0.105 & -0.138 & -0.287 & -0.292 & -0.166 & -0.087 \\
\midrule
14 & -8.597 & -2.626 & 0.006 & 0.327 & 0.289 & 0.140 \\
\midrule
15 & -0.522 & -0.214 & -0.273 & -0.299 & -0.173 & -0.094 \\
\midrule
16 & -11.196 & -5.194 & -1.990 & -0.813 & -0.405 & -0.217 \\
\midrule
17 & -3.828 & -1.192 & -0.320 & -0.192 & -0.105 & -0.058 \\
\midrule
18 & -6.869 & -2.392 & -0.644 & 0.007 & -0.002 & 0.001 \\
\midrule
19 & 0.092 & -0.299 & -0.181 & -0.138 & -0.074 & -0.035 \\
\midrule
20 & -5.639 & -0.308 & 0.923 & 0.796 & 0.448 & 0.234 \\
\midrule
21 & 0.399 & -0.968 & -0.796 & -0.509 & -0.275 & -0.145 \\
\midrule
22 & -4.474 & 3.712 & 4.440 & 3.033 & 1.700 & 0.880 \\
\midrule
23 & 0.456 & -0.901 & -0.703 & -0.533 & -0.281 & -0.145 \\
\midrule
24 & -5.863 & 4.241 & 5.617 & 3.797 & 2.090 & 1.074 \\
\midrule
25 & 1.433 & -0.464 & -0.774 & -0.632 & -0.347 & -0.182 \\
\midrule
26 & -5.034 & 4.908 & 6.328 & 4.362 & 2.459 & 1.271 \\
\midrule
27 & 0.661 & -0.914 & -1.156 & -0.931 & -0.530 & -0.284 \\
\midrule
28 & -15.667 & 4.763 & 9.616 & 6.975 & 4.062 & 2.096 \\
\midrule
29 & -9.747 & 1.663 & 3.836 & 2.455 & 1.384 & 0.673 \\
\midrule
30 & -0.151 & 16.639 & 15.387 & 9.638 & 5.334 & 2.769 \\
\midrule
31 & -5.403 & 0.917 & 1.957 & 1.555 & 1.009 & 0.574 \\
\midrule
32 & 17.796 & 6.777 & 3.728 & 3.231 & 2.507 & 1.583 \\
\midrule
33 & -4.935 & -2.141 & 2.055 & 2.931 & 2.395 & 1.561 \\
\bottomrule
\end{tabular}
\end{threeparttable}
}
\end{table}

\newpage
\clearpage

\begin{table}[!htbp]
\caption{Comparison of firing counts percentage difference between the baseline and our proposed TPP method for VGG-16 on ImageNet using QCFS.}
\label{tab:spike-counts-cifar100-vgg16-imagenet}
\centering
\scalebox{0.85}
{
\begin{threeparttable}
\begin{tabular}{@{}ccccccc@{}}
\toprule
Layer & T=4 & T=8 & T=16 & T=32 & T=64 & T=128 \\ 
\toprule
1 & 5.487 & 2.776 & 1.444 & 0.712 & 0.363 & 0.179 \\
\midrule
2 & 0.418 & 0.173 & -0.005 & 0.007 & 0.007 & 0.006 \\
\midrule
3 & -2.375 & -0.883 & -0.351 & -0.128 & -0.062 & -0.031 \\
\midrule
4 & 6.170 & 2.181 & 0.627 & 0.121 & 0.024 & -0.002 \\
\midrule
5 & -3.338 & -0.318 & 0.327 & 0.306 & 0.173 & 0.097 \\
\midrule
6 & 7.036 & 2.769 & 0.993 & 0.385 & 0.173 & 0.078 \\
\midrule
7 & -5.722 & -3.482 & -1.661 & -0.800 & -0.400 & -0.200 \\
\midrule
8 & -6.155 & 0.310 & 1.411 & 0.955 & 0.507 & 0.269 \\
\midrule
9 & -0.718 & 1.172 & 0.725 & 0.337 & 0.162 & 0.081 \\
\midrule
10 & -12.833 & -9.060 & -4.882 & -2.359 & -1.145 & -0.564 \\
\midrule
11 & 12.966 & 11.241 & 7.718 & 4.443 & 2.344 & 1.188 \\
\midrule
12 & -11.194 & -14.874 & -12.032 & -7.889 & -4.437 & -2.395 \\
\midrule
13 & -37.388 & -30.782 & -20.701 & -12.296 & -6.527 & -3.377 \\
\midrule
14 & -23.619 & -12.312 & -3.929 & -0.233 & 0.585 & 0.382 \\
\midrule
15 & -10.988 & -18.476 & -13.953 & -7.904 & -4.091 & -2.015 \\
\bottomrule
\end{tabular}
\end{threeparttable}
}
\end{table}



\newpage
\section{Membrane potential Distribution}
\label{appendix:mem-pot-dist}

\begin{figure}[!htbp]
\centering
\includegraphics[width=\textwidth]{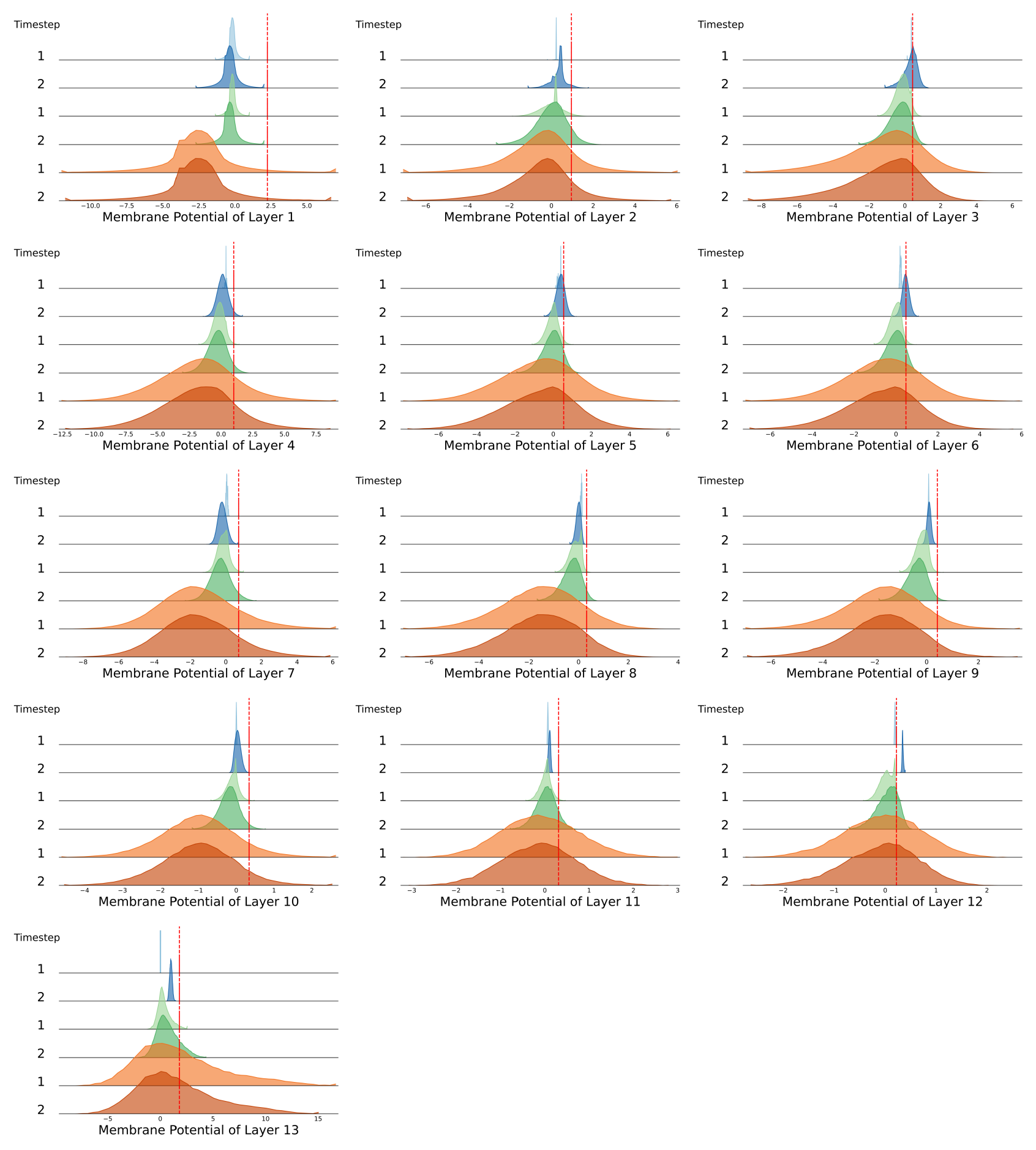}
\caption{The membrane potential distributions of the first channel (randomly selected) across three modes (baseline, shuffle, and probabilistic) in VGG-16 on CIFAR-100. For comparison, the first two timesteps (t=1, t=2) from a total of eight timesteps (T=8) are selected for each mode. The baseline mode (blue) achieves an accuracy of 24.22\%, while the shuffle mode (light green) improves accuracy to 70.54\%, and the probabilistic mode (dark orange) further increases accuracy to 73.42\%. The distributions are shown before firing, and the red dashed line indicates the threshold voltage (Vth) for the layer.}
\label{fig-original-prob-membrane-pot}
\end{figure}

\newpage
\clearpage

\begin{figure}[!htbp]
\centering
\includegraphics[width=\textwidth]{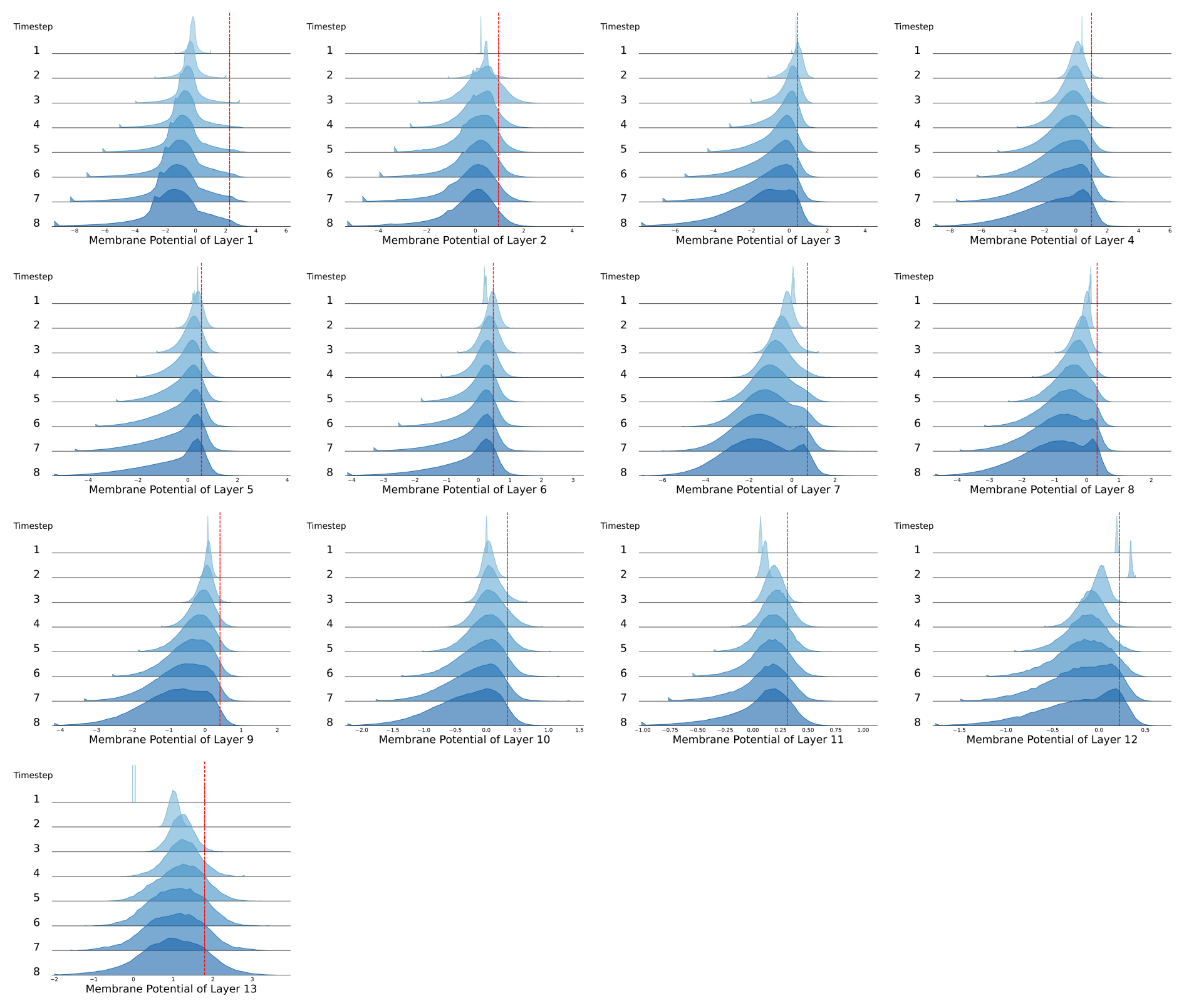}
\caption{The membrane potential of the first channel (randomly selected) from layer 1 in SNNC baseline mode using VGG-16 on CIFAR-100 achieves an accuracy of 24.22\% before firing. }
\label{fig-baseline-membrane-pot}
\end{figure}

The first two timesteps exhibit an abnormal distribution compared to those at t=4 to t=8. This discrepancy arises from the initially incorrect membrane potential before firing, which affects the firing rate and propagates errors layer by layer. A detailed quantifiable error analysis is provided in Appendix Section~\ref{appendix-sec:conversion_error_analysis}. Furthermore, as shown in Figure~\ref{fig-shuffle-membrane-pot}, shuffling the membrane potential effectively alleviates this effect.

\begin{figure}[!htbp]
\centering
\includegraphics[width=\textwidth]{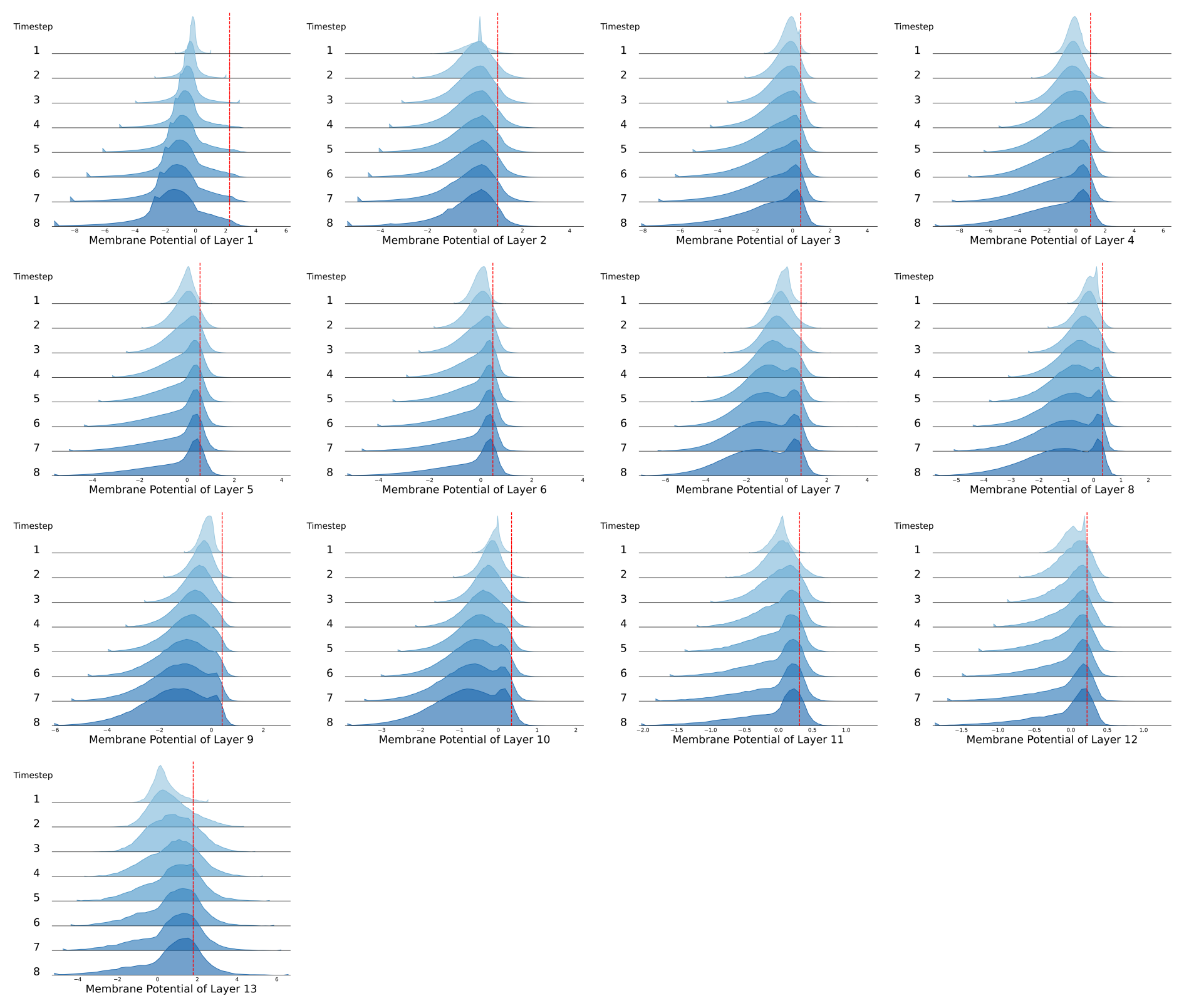}
\caption{Membrane potential of the first channel (randomly selected) before firing in SNNC shuffle mode using VGG-16 on CIFAR-100. The achieved accuracy is 70.54\%, indicating the impact of random spike rearrangement.}
\label{fig-shuffle-membrane-pot}
\end{figure}

\begin{figure}[!htbp]
\centering
\includegraphics[width=\textwidth]{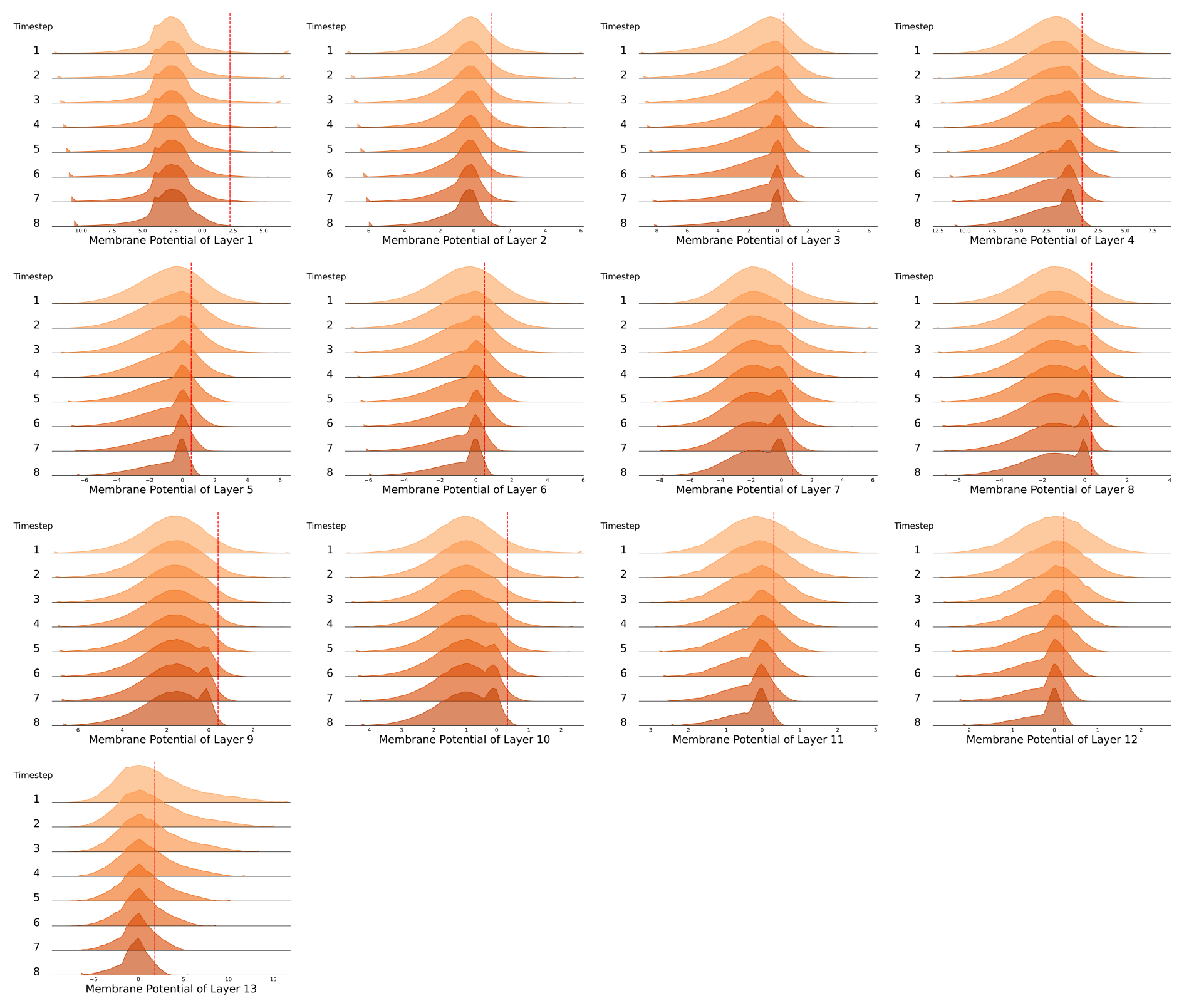}
\caption{Membrane potential of the first channel (randomly selected) before firing in SNNC probabilistic mode using VGG-16 on CIFAR-100. The accuracy increases to 73.42\%.}
\label{fig-prob-membrane-pot}
\end{figure}


\newpage
\clearpage

\section{Permutations and ANN-SNN conversion}\label{app permutations}

\textbf{Heuristics behind permutations} We come back to the original motivation, and the mysterious effect of temporal misalignment. To this end, we notice that permutations may act as a ``uniformizer'' of the inputs to the spiking neuron, which is highly related to notions of phase lag or unevenness of the inputs (see \cite{li2022bsnn} and \cite{bu2022optimal}, respectively). 
\begin{restatable}{theorem}{exppermutations}\label{thm permutations expected} Suppose we have $N$ spiking neurons that produced spike trains $s_i[1], s_i[2],\dots, s_i[T]$, $i=1,\dots,N$. Furthermore, suppose that these spike trains are modulated with weights $w_1,\dots,w_N$, and as such give input to a neuron (say from the following layer) in the form $x[t]=\sum w_i s_i[t]$, for $t=1,\dots, T$. For a given permutation $\pi = (\pi_1,\dots,\pi_N)$, 
let $\pi s_i$ denote the permutation of the spike train $s_i$. Then, for every $t_1,t_2\in \{1,2,\dots, T\}$, 
$$
E_{\pi}[\sum w_i\pi s_i[t_1]]=E_{\pi}[\sum w_i\pi s_i[t_2]].
$$
\end{restatable}
\begin{proof}
It is enough to prove that for each $i=1,\dots,N$, 
\begin{equation}\label{eq exp perm}
    E_\pi[s_i[t_1]] = E_\pi[s_i[t_2]].
\end{equation}
Let $A(t_i)$ be the cardinality of the set of all the permutations that end up with a spike in step $t_i$, and note that the probability of having a spike at $t_i$ is then $\frac{A(t_i)}{T!}$. But, for each permutation that ends up with a spike at $t_i$, one can find a permutation that ends up with a spike at $t_2$ (by simply applying a cyclic permutation) and moreover this correspondence is bijective. In particular $A(t_i)$ is independent of $i$. The equation \eqref{eq exp perm} and the statement follow.
\end{proof}
The previous result deals with the expected outputs with respect to the permutations. When it comes to the action of a single permutation, we make the following observation. The effect of a single permutation is mostly visible on spike trains that have a \textbf{low number of spikes}. This, in turn, is related to the situation where the input to the neuron is low throughout time, and it takes longer for a neuron to accumulate enough potential in order to spike, hence the neuron spikes at a later time during latency. In this case, a single permutation of the output spike(s) actually move the spikes forward in time (in general) and as such contributes to the elimination of the unevenness error, which appears when the input to a neuron in the beginning is higher than the average input through time (hence, the neuron produces superfluous spikes in the beginning, which shouldn't be the case).

\begin{table}[H]
\caption{Recorded accuracy after $t\leq T$ time steps, when the baseline model is "permuted" in latency $T$. Setting is VGG-16, CIFAR-100.}
\label{tab: stabilization}
\centering
\scalebox{0.85}
{
\begin{threeparttable}
\begin{tabular}{@{}ccccccccc@{}}
\toprule
 Method & ANN & t=1 & t=2 & t=4 & t=8 & t=16 & t=32  \\ 
\toprule
 QCFS~\cite{bu2022optimal} &  & 49.09 & 63.22 & 69.29 & 73.89 & 75.98 & 76.52 &  \\
 \textbf{Ours (Permute)} & T=4 & 68.11 & 71.91 & 74.2 &  &  &  & \\
 \textbf{Ours (Permute)} & T=8 & 71.76 & 74.11 & 75.53 & 75.86 &  &  & \\
 \textbf{Ours (Permute)} & T=16  & 72.75 & 74.27 &	75.63 & 76.0 & 76.39 &  & \\
 \textbf{Ours ((Permute)} & T=32 & 73.15 & 75.23 & 75.74 & 76.27 & 76.59 & 76.52 & \\
\midrule
 RTS~\cite{deng2021optimal} &  & 1.0 & 1.03 & 23.76 & 43.81 & 56.23 & 67.61 &  \\
 \textbf{Ours (Permute)} & T=4 & 22.9 & 30.78 &	34.54 &  &  &  & \\
\textbf{Ours ((Permute)} & T=8 & 45.11 & 52.7 &	59.2 & 62.58 &  &  & \\
 \textbf{Ours ((Permute)} & T=16  & 54.58 & 64.37 &	68.6 & 70.8 & 71.79 &  & \\
 \textbf{Ours (Permute)} & T=32 & 62.76 & 69.12 & 71.76 & 73.31 & 74.09 & 74.6 & \\
\bottomrule
\end{tabular}
\end{threeparttable}
}
\end{table}

\textbf{Remarks:} 
\begin{enumerate}
    \item In Table \ref{tab: stabilization} we combine permutations with baseline models in fixed latency $T$. Afterwards, we record the accuracies of such "permuted" model for lower latencies $t$. We can notice a sharp increase in the accuracies compared to the baselines, and in particular, the variance in accuracies across $t$ is reduced. 
    \item \textbf{Baseline analysis:} 
    \begin{enumerate}
    \item SNN models converted from a pretrained ANN aim to approximate the ANN activation values with firing rates. In particular, in lower time steps, the approximation is too coarse as the firing rate has only few possibilities to use to approximate the ANN (continuous) values. For example, in $T=1$, the baselines are attempting to approximate ANN activations with binary values $0$ and $\theta$.
    \item Moreover, at each spiking layer, the spiking neurons at early time steps, use only the outputs of the previous spiking layer from the same, early, time steps. As this information is already too coarse, \textbf{the approximation error accumulates throughout the network}, finally yielding in models that are underperfoming in low latencies.
    \item With longer latencies, the model is using more spikes and is able to approximate the ANN values more accurately, and to correct the results from the first time steps. 
    \end{enumerate}
    \item \textbf{Effect of permutations:}
    \begin{enumerate}
        \item When performing permutations on spike trains after spiking layers in the baseline models, the input to the next spiking layer in lower time steps, \textbf{no longer depends only on the outputs of the previous layer in the same lower time steps, but it depends on the outputs in all time steps $T$}. 
        \item In particular, when spiking layer is producing spikes at time step $t=1$, it does so "taking into account" (via permutation) outputs at all the time steps from the previous spiking layer.
        \item As a way of example, consider two spiking neurons $N_1$ and $N_2$, where $N_2$ receives the weighted input from $N_1$. If a spiking neuron $N_1$ in one layer has produced spike train $s = [1,0,0,0]$, in approximating ANN value of $.25$, then a spiking neuron $N_2$ at the first time step will use 1 as the approximation and will receive the input $W\cdot 1$ from neuron $N_1$. However, after a generic permutation of $s$, the probability of having zero at the first time step of output of neuron $N_1$ is $\frac{3}{4}$ (as oppose to having 1 with probability $\frac{1}{4}$), and at the first time step neuron $N2$ will most likely receive the input $W\cdot 0=0$ from neuron $N_1$, which is a rather better approximation for $W\cdot .25$ than $W$ itself.
        \item This property of receiving input at lower $t$ but taking into account the previous layer spike outputs at all the time steps is not only exclusive to lower $t$. Indeed, at every time step $t\leq T$, the input at a spiking layer is formed by taking into account spiking train outputs from the previous layer at all the time steps, but having already accounted for for the observed input at the first $t<1$ steps. 
        \item In general, the permutations overall increase the performance of the baselines because the spike trains are "uniformized" in accordance to their rate, and the accumulation error is reduced. If a layer $l$ has produced spike outputs that well approximate the $l$ layer in ANN, then, after a generic permutation, at each time step starting with the first, the next layer is receiving the most likely binary approximation of those rates. 
        \item This is nothing but Theorem \ref{thm permutations expected} in visible action. 
        \item Besides Table \ref{tab: stabilization}, we provide further evidence on how permutation affect the baselines through the observed membrane potential in the following sections. 
    \end{enumerate}
\end{enumerate}

\subsection{The effect of permutations on performance: Further experiments}



\begin{figure}[!htbp]
\centering
\includegraphics[width=\textwidth]{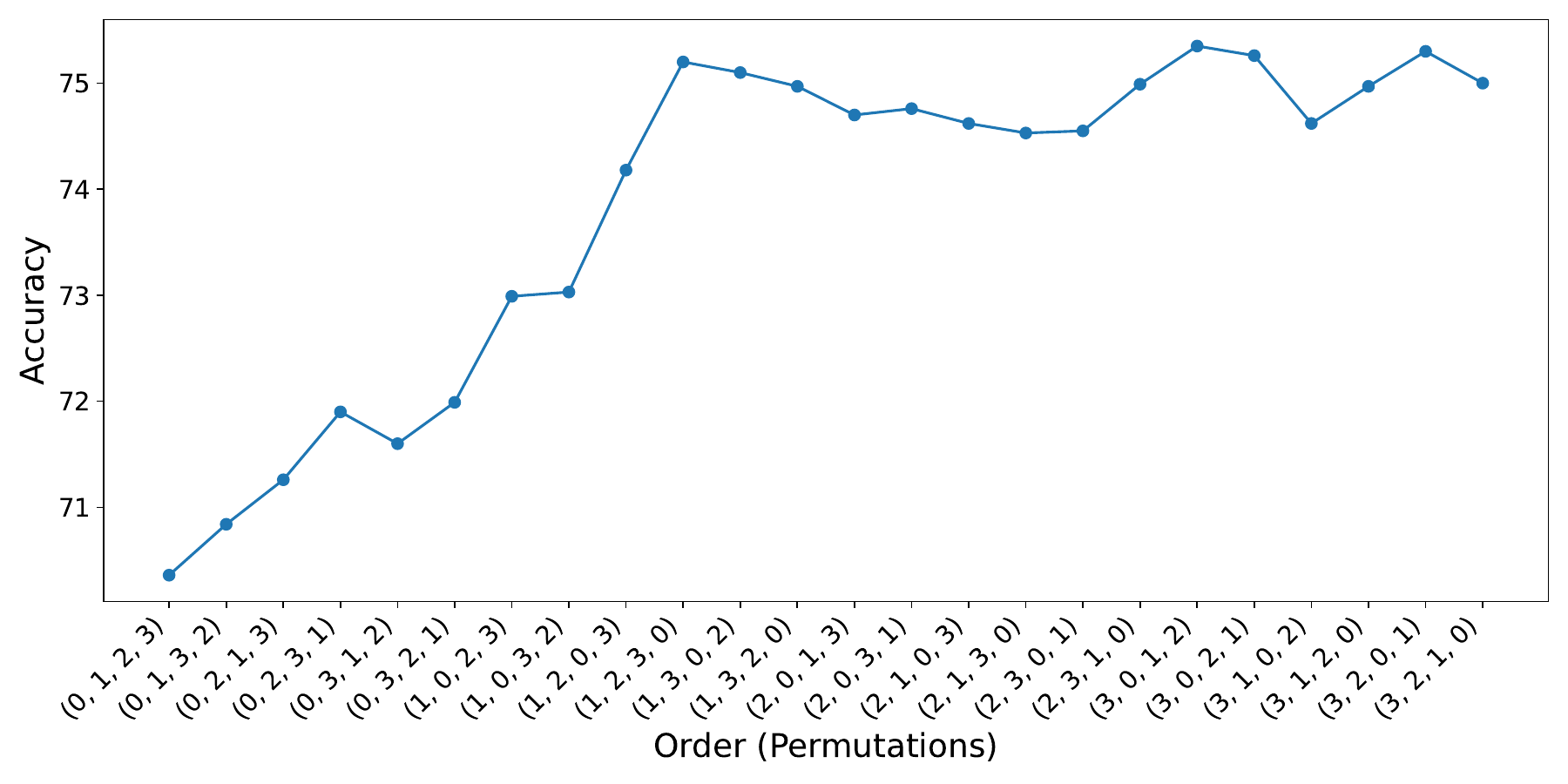}
\caption{Accuracy comparison for all \( T! \) permutations of input order over \( T=4 \) time steps using QCFS with VGG-16 on CIFAR-100. Results of permuted orders outperform the original, non-permuted order $(0,1,2,3)$. Baseline accuracy is $69.31\%$, The ANN accuracy is $76.21\%$.}
\label{fig-order-acc-1}
\end{figure}

\begin{figure}[!htbp]
\centering
\includegraphics[width=\textwidth]{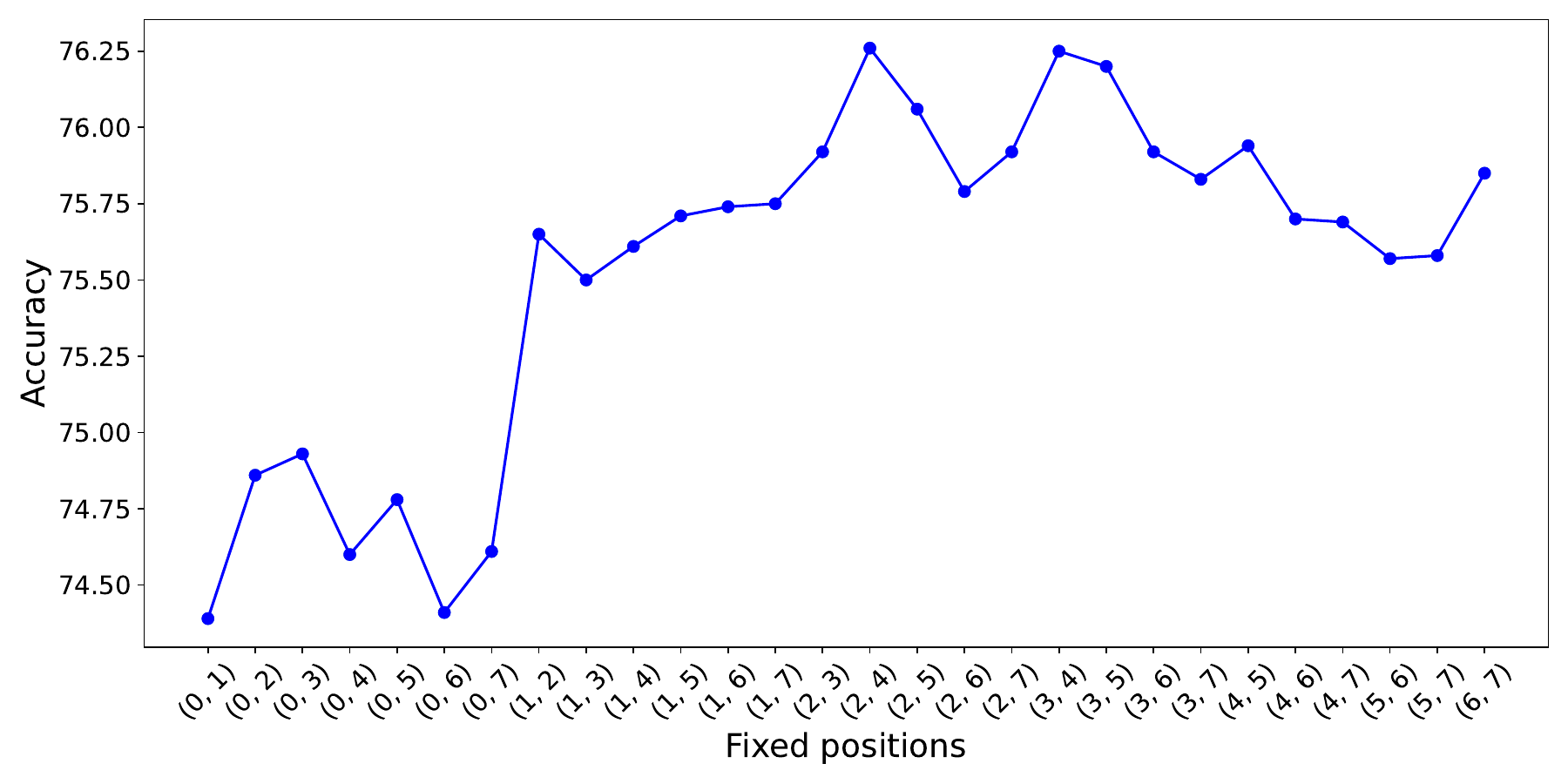}
\caption{Accuracy comparison for permutations over 8 time steps, fixing given pairs of time steps. Setting is VGG-16, CIFAR-100. The baseline (QCFS) accuracy is $73.89\%$, ANN accuracy is $76.21\%$.}
\label{fig-order-acc-2}
\end{figure}


\begin{figure}[!htbp]
\centering
\includegraphics[width=.7\textwidth]{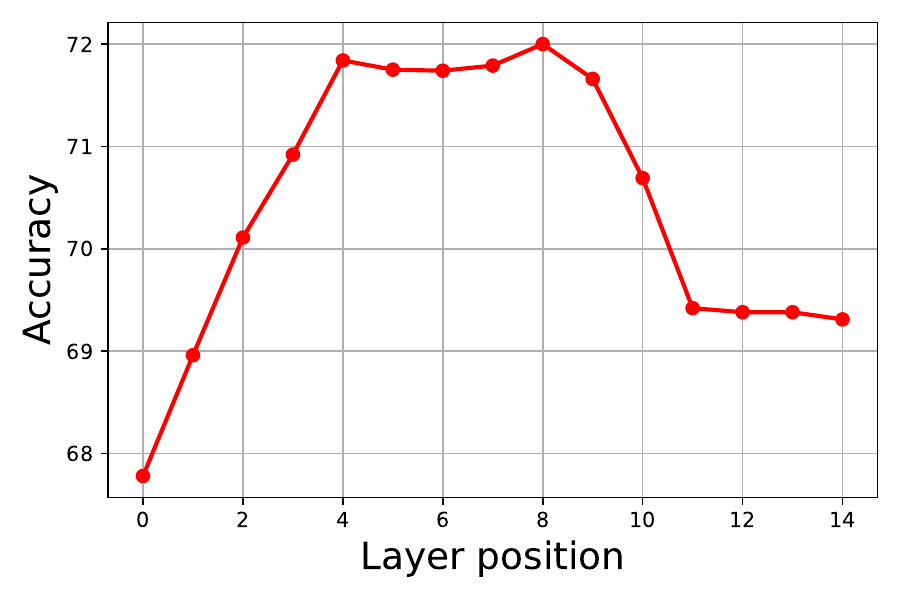}
\caption{Accuracy of the model when a permutation is applied on a single layer using QCFS baseline. Setting is VGG-16, $T=4$, CIFAR-100. Baseline accuracy is $69.31\%$, ANN accuracy is $76.31\%$}
\label{fig-order-acc-3}
\end{figure}

\begin{figure}[!htbp]
\centering
\includegraphics[width=.7\textwidth]{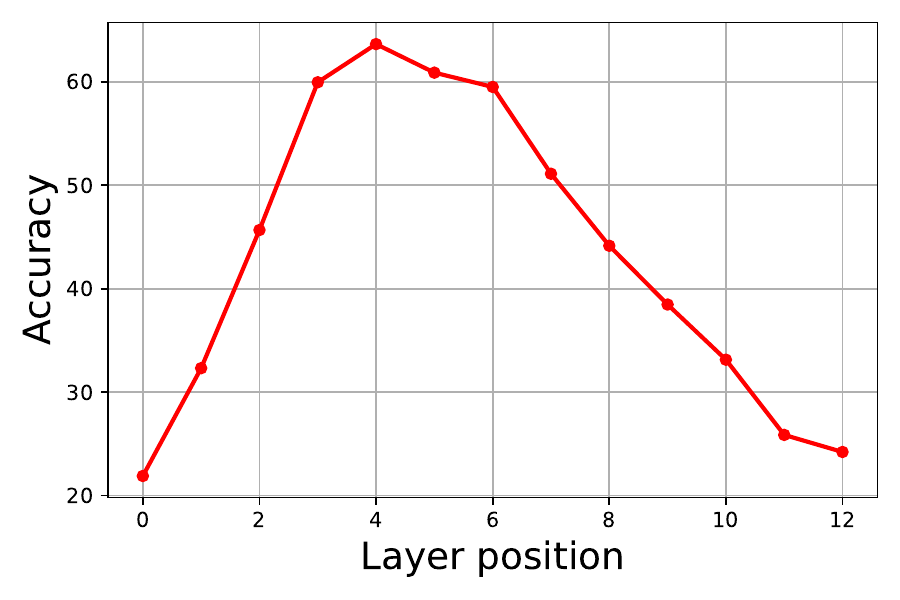}
\caption{Accuracy of the model when a permutation is applied on a single layer using SNNC baseline. Setting is VGG-16, $T=8$, CIFAR-100. Baseline accuracy without calibration is $24.22\%$, ANN accuracy is $77.87\%$}
\label{fig-order-acc-4}
\end{figure}

\end{document}